\documentclass[twoside]{article}

\usepackage[accepted]{aistats2021}
\usepackage{multicol,lipsum,environ}

\NewEnviron{auxmulticols}[1]{%
  \ifnum#1<2\relax
    \BODY
  \else
    \begin{multicols}{#1}
      \BODY
    \end{multicols}%
  \fi
}

\usepackage[dvipdfm]{graphicx}
\usepackage{bmpsize}

\usepackage[T1]{fontenc}

\usepackage[margin=1in]{geometry}
\usepackage{algorithm}
\usepackage{algpseudocode}
\usepackage[dvipsnames]{xcolor}
\usepackage{subfigure}
\usepackage[colorlinks=true,linkcolor=blue,citecolor=blue]{hyperref}%

\usepackage{amsmath,amsthm,amssymb}
\usepackage{cleveref}
\usepackage{thmtools,thm-restate}
\usepackage{caption}
\usepackage{mathtools}
\usepackage{xfrac}
\usepackage{enumitem}

\newcommand{\suchthat}{\;\ifnum\currentgrouptype=16 \middle\fi|\;}

\DeclarePairedDelimiter\ceil{\lceil}{\rceil}

\newcommand{\R}{\mathbb{R}}

\newcommand{\Reg}{\text{Reg}}

\newcommand{\hatf}{\hat{f}}
\newcommand{\hata}{\hat{a}}

\newcommand{\F}{\mathcal{F}}

\newcommand{\E}{\mathop{\mathbb{E}}}
\newcommand{\A}{\mathcal{A}}

\newcommand{\Xscript}{\mathcal{X}}
\newcommand{\ordO}{\mathcal{O}}
\newcommand{\ordOt}{\tilde{\mathcal{O}}}
\newcommand{\Unif}{\text{Unif}}
\newcommand{\I}{\mathbb{I}}

\newcommand{\G}{\mathcal{G}}

\newcommand{\hatG}{\hat{\mathcal{G}}}
\newcommand{\hatg}{\hat{g}}
\newcommand{\erisk}{\mathcal{E}}
\newcommand{\haterisk}{\hat{\mathcal{E}}}
\newcommand{\event}{\mathcal{W}}
\newcommand{\activesteps}{\mathcal{T}_{\text{active}}}
\newcommand{\passivesteps}{\mathcal{T}_{\text{passive}}}

\newcommand{\comp}{\textbf{comp}}
\newcommand{\Stilde}{\widetilde{S}}
\newcommand{\Zscript}{\mathcal{Z}}

\newcommand{\GeneralizedFalcon}{\text{Epsilon-FALCON}}

\newtheorem{assumption}{Assumption}

\newtheorem{lemma}{Lemma}
\newtheorem{corollary}{Corollary}

\begin{document}

%

%

\twocolumn[

\aistatstitle{Tractable contextual bandits beyond realizability}

\aistatsauthor{ Sanath Kumar Krishnamurthy \And Vitor Hadad \And  Susan Athey }

\aistatsaddress{ Stanford University \And  Stanford University \And Stanford University } ]

\begin{abstract}
  Tractable contextual bandit algorithms often rely on the realizability assumption -- i.e., that the true expected reward model belongs to a known class, such as linear functions. In this work, we present a tractable bandit algorithm that is not sensitive to the realizability assumption and computationally reduces to solving a constrained regression problem in every epoch. When realizability does not hold, our algorithm ensures the same guarantees on regret achieved by realizability-based algorithms under realizability, up to an additive term that accounts for the misspecification error. This extra term is proportional to T times a function of the mean squared error between the best model in the class and the true model, where T is the total number of time-steps. Our work sheds light on the bias-variance trade-off for tractable contextual bandits. This trade-off is not captured by algorithms that assume realizability, since under this assumption there exists an estimator in the class that attains zero bias.
\end{abstract}

\section{Introduction}
\label{sec:introduction}

Contextual bandit algorithms serve as a fundamental tool for online decision making and have been used in a wide range of settings from recommendation systems \cite{agarwal2016making} to mobile health \cite{tewari2017ads}, and due to their applicability over the past couple of decades there has been an increasing amount of research in contextual bandits \cite{lattimore2020bandit}. However, the performance of many common algorithms relies on an assumption called ``realizability'', which requires the analyst to possess some knowledge about the underlying data generating process -- and often also relies on some luck that the process be relatively simple. When this assumption is satisfied, there exist algorithms that are statistically optimal and computationally tractable (in a sense we'll discuss more below). However, when it is violated, the performance of these algorithms can degrade in unexpected ways. The search for tractable algorithms that do not rely on this assumption an ongoing open problem \cite{foster2019model}.  In this work, we propose an algorithm that is optimal when the ``realizability'' assumption is satisfied and whose behavior is accurately characterized in its absence. We will also point to directions of research that may help do away with this assumption entirely.

Our underlying setup is the general stochastic contextual bandit setting. Using potential outcome notation, observations are represented as a sequence of iid random variables $(x_t, r_t)$, where $x_t \in \mathcal{X}$ stands for a context in arbitrary set $\mathcal{X}$ and $r_t \in [0, 1]^{K}$ is a vector of rewards, where $K : = |\mathcal{A}|$ is the (finite) number of actions. Upon selecting one of action $a_t \in \mathcal{A}$, the algorithm observes $r_t(a_t)$. Therefore, the sequence of observed data points is $(x_t, a_t, r_t(a_t))$. Here $t$ denotes the time-step which and is also the index for the sequence of observations. This sequence has length $T$, which may be known or unknown. A ``policy'' is a deterministic mapping from contexts to actions, representing a particular action selection strategy. Relative to the set of all policies $\mathcal{A}^{\mathcal{X}}$, we define the optimal policy as $\pi^{*} = \arg\max_{\pi \in \mathcal{A}^{\mathcal{X}}} \E_{x_t,r_t}[r_t(\pi(x_t))]$, where the expectation is taken over contexts and rewards.\footnote{Uniqueness of the optimal policy is not important for our results.} A ``reward model'' or ``outcome model'' is a function that (potentially inaccurately) represents the conditional expectation of potential outcomes given action and context. Reward models will often be represented as $f(x,a)$. We say that a reward model ``induces a policy $\pi$'' if $\pi(x) \in \arg\max_{a} f(x, a)$ for every $x$.

The goal of bandit algorithms is to find a sequence of actions that maximizes the sum of rewards observed during the experiment or, equivalently, to minimize cumulative regret, defined as the difference between the reward that was observed and that that would have been observed under the optimal policy,
\begin{align}
  \label{eq:regret}
  R_T := \sum_{t=1}^{T} r_t(\pi^*(x_t)) - r_t(a_t).
\end{align}
The statistical performance of different algorithms is characterized by the rate at which \eqref{eq:regret} grows with the length of the experiment $T$.

Contextual bandit algorithms can often be categorized into three groups, depending on what is assumed about the underlying data-generating process. The first group of algorithms are the ``agnostic'' algorithms. These algorithms make no assumptions about the reward model, and they learn the best policy in some fixed class $\Pi \subseteq \mathcal{A}^{\mathcal{X}}$ while balancing the exploration-exploitation trade-off. To do this, these algorithms \cite{beygelzimer2011contextual, dudik2011efficient, agarwal2014taming} need to construct a distribution over the policies $\Pi$ in every epoch. Constructing this distribution is computationally challenging, and hence this approach is colloquially referred to as the ``Monster'' \cite{langford2014interactive}. We now focus on the results in \cite{agarwal2014taming} because computationally and statistically, they provide the state of the art agnostic algorithms. When $\Pi$ is a finite class, \cite{agarwal2014taming} present an algorithm called ILTCB that constructs a distribution with support size of $\ordO(\log|\Pi|)$ and the regret \footnote{Note that while this notion of regret compares against the best policy in $\Pi$, the notion of regret used in this paper compares against the true optimal policy $\pi^*$.} against the best policy in $\Pi$ scales at the rate $\ordOt(KT\log|\Pi|)$. Each policy in the support of this distribution can be computed by solving the following cost-sensitive classification problem:
\begin{align}
  \label{eq:cost_sensitive}
  \arg\max_{\pi \in \Pi} \sum_{s=1}^{t} \hat{r}_s(\pi(x_s)),
\end{align}
where $(x_s,\hat{r}_s)$ is some sequence in $\Xscript\times[0,1]^K$. When $\Pi$ is large, the support of the distribution needed to be computed in every epoch of ILTCB may be large and hence would still be impractical to implement. To overcome this limitation, \cite{agarwal2014taming} propose a heuristic called Online Cover (with parameter $l$) that computes a distribution over polices using the same approach as ILTCB but stops increasing the support of this distribution after computing some fixed number of policies $l$ for the support. To the best of our knowledge there aren't any theoretical guarantees for Online Cover. Further, finding an exact solutions to \eqref{eq:cost_sensitive} is generally intractable, so implementations of Online Cover use heuristics to solve this optimization problem.  

The second group of algorithms requires knowledge about some set of functions $\mathcal{F}$ that is assumed to include the true reward model. That is, that there exists a function $f^* \in \mathcal{F}$ such that $f^*(x, a) = \E_{x_t,r_t}[r_t(a) | x_t=x]$ for all contexts and actions. This assumption is called ``realizability'', and it often allows for algorithms that computationally tractable and typically easier to implement. Computationally, algorithms in this category rely on being able to solve the regression problem
\begin{align}
  \label{eq:regression_oracle}
  \hatf_t = \arg\min_{f \in \mathcal{F}} \sum_{s=1}^{t-1} (f(x_s, a_s) - r_s(a_s))^2,
\end{align}
or a weighted version of it, either online or offline. A routine that solves \eqref{eq:regression_oracle} is called ``regression oracle''. This class includes algorithms built on upper confidence bounds \cite{li2010contextual, abbasi2011improved, foster2018practical} or Thompson sampling \cite{agrawal2013thompson, russo2018tutorial}, and algorithms built on simple probabilistic selection strategies \cite{abe1999associative, foster2020beyond, simchi2020bypassing}. Regret rates for this class of algorithms are related to the complexity class of the outcome model $\mathcal{F}$, and under realizability optimal algorithms attain a rate of $\ordOt(\sqrt{TK\log|\F|})$ for finite classes $\F$ (similar results are available for more general classes). In particular, the FALCON  algorithm of \cite{simchi2020bypassing} will serve as the basis of our method attains this statistically optimal rate (so long as realizability holds) and is computationally tractable, in that the algorithm only needs to solve the problem \eqref{eq:regression_oracle} a small (at most logarithmic) number of times during the experiment.

\begin{auxmulticols}{1}
\begin{figure*}[ht]
  \centering
  \includegraphics[width=\textwidth]{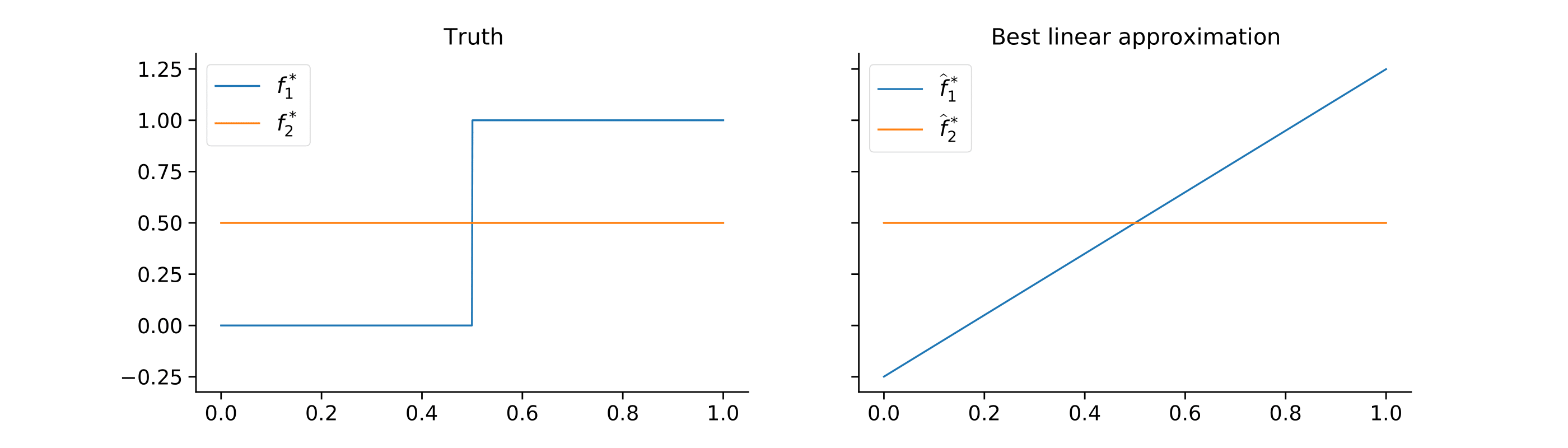}
  \vspace{.1in}
  \caption{\textbf{Example}. Data-generating process (left) and its best approximation in the class of linear functions (right). The induced policies are the same.}
  \label{fig:example_truth_blp}
\end{figure*}
\end{auxmulticols}

\begin{auxmulticols}{1}
\begin{figure*}[ht]
  \centering
  \vspace{.1in}
  \includegraphics[width=\textwidth]{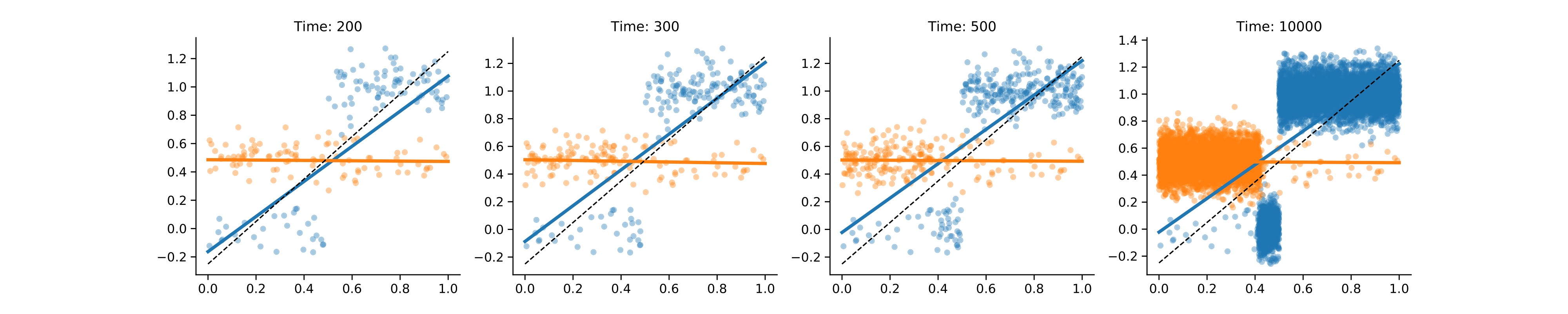}
  \vspace{.1in}
  \caption{\textbf{Evolution of model estimates} for the UCB example. Although the model is initially correct, the distribution shift caused by the assignment mechanism ends up biasing reward estimates over time. Best linear approximation in black, dotted line. In our proposed algorithm, the blue and black lines are kept appropriately close, preventing this phenomenon.}
  \label{fig:example_evolution}
\end{figure*}
\end{auxmulticols}

A third set of bandit algorithms that does not fall neatly into any of the two categories above are algorithms that allow for a non-parametric model class. For example, in \cite{rigollet2010nonparametric}, \cite{perchet2013multi} the reward model is assumed to be H\"{o}lder continuous but non-differentiable, and in \cite{hu2020smooth}, \cite{gur2019smoothness} it satisfies a H\"{o}lder smoothness assumption. The main characteristic of this class of algorithms is that they partition the covariate space into hypercubes of appropriate size and run multi-armed bandit algorithms within each cube. Depending on the smoothness of the reward model, there can be some information sharing across cubes that induces correlation across assignments in adjacent hypercubes and decreases regret. Although this is a very interesting direction of research, the structure of these algorithms forces the running time to exponentially depend on the context dimension, making them computationally intractable and hard to implement for most real life problems. Hence, for the rest of the paper, we will focus on the first two classes of algorithms.

\subsection{The problem with realizability}
\label{sec:problem_with_realizability}

As we have mentioned before, realizability is an extremely convenient and pervasive assumption in many tractable contextual bandit algorithms, but it is nevertheless very strong. In this section, we attempt to shed light on some issues that may arise in its absence.

To fix ideas, start from the following illustration. There are two actions and a single context is distributed uniformly on the unit interval. However, unbeknownst to the researcher, the conditional average rewards for each arm are a step function $f_1^*(x_t) := \mathbb{I}\{ x_t > 0.5 \}$ and a constant $f_2^*(x_t) \equiv 0.5$ (Figure \ref{fig:example_truth_blp}).  Rewards are observed with error $\epsilon_t \sim \mathcal{N}(0, .01)$. The researcher erroneously assumes that both can be realized in the class $\mathcal{F}$ of linear functions. Fortunately, in this example the best linear approximation $(\hatf_1^*, \hatf_2^*)$ induces a good policy. In fact, it coincides with the one the researcher would obtain if they had knowledge about the true function class -- i.e., the policy induced by the best linear approximation $\hat{\pi}^*$ defined by $\hat{\pi}(x) = 1$ if $x > 0.5$ and $\hat{\pi}(x) = 2$ otherwise actually coincides with the policy induces by the true model $\pi^*(x) := \arg\max_{a} f^*(x,a)$. Therefore, if the sequence of fitted models $(\hatf_{1,t}, \hatf_{2,t})$ converges to the best linear approximation $(\hatf_{1,t}^*, \hatf_{2,t}^*)$, regret should decay to zero asymptotically. However, as we will see next, this convergence may not happen.

Let us assume that the researcher collects data via LinUCB \cite{li2010contextual}, with model updates in batches of 100 observations. Figure \ref{fig:example_evolution} shows the evolution of the estimated models $(\hatf_{1,t}, \hatf_{2,t})$ over time for a single simulation. After about a few hundred observations, the estimated model approximates the best linear approximation well and regret is small since the induced policy is nearly optimal. However, by continuing to assign treatments following this policy (plus some negligible exploration), the distribution of observations changes, which pushes the model away from the best linear approximation. In turn, this causes per-period regret to increase over time, as we show on Figure \ref{fig:regret_increase}.

\begin{figure}[H]
  \centering
  \includegraphics[width=.45\textwidth]{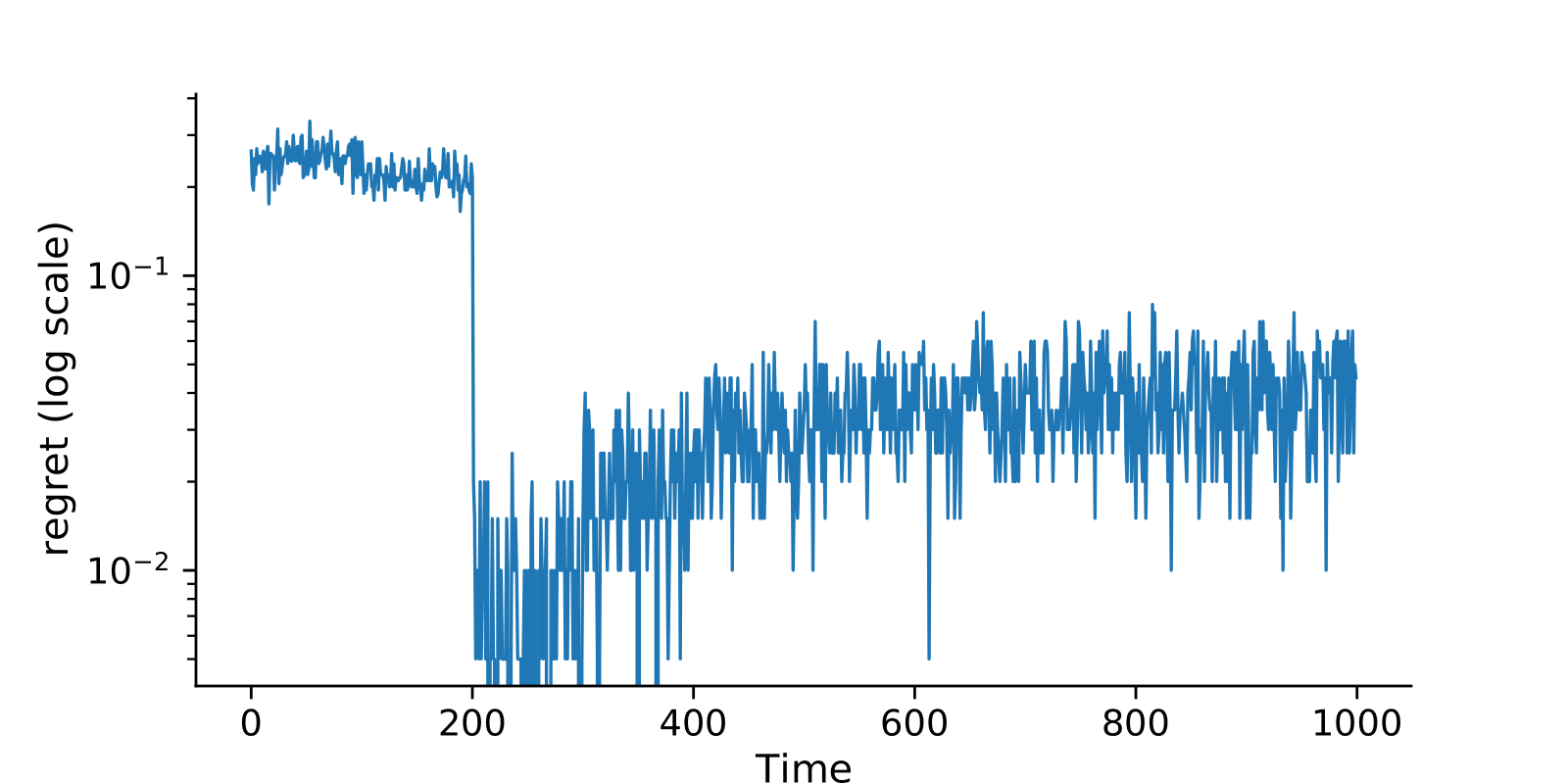}
  \vspace{.1in}
  \caption{\textbf{Evolution of per-period regret} for the UCB example described in the text. Regret initially decreases because the estimated models $(\hatf_{1,t}, \hatf_{2,t})$ are close to the best linear approximations $(\hatf^*_1, \hatf^*_2)$, which in this example induces a good policy. However, as we gather more data that was collected with an exploitation objective, the estimated model diverges and the performance of the algorithm degrades. (Average across 100 simulations.)}
  \label{fig:regret_increase}
\end{figure}

 The previous example demonstrates that in the absence of realizability the dynamics of adaptive data collection can lead the algorithm to learn a policy that is suboptimal relative to the one that it would have learned under non-adaptive data collection. As an extreme thought example, one may also consider a situation in which actions are assigned via the optimal policy $\pi^*$. If we were to fit a linear model using exclusively this data, we would estimate that $\hatf_1 \equiv 1$ and $\hatf_2 \equiv 0.5$, which would in turn induce the policy $\pi(x) \equiv 1$ -- a policy that always assigns arm $1$ everywhere and therefore clearly suboptimal. In fact, more can be said. We can construct examples where even when the approximation error $b$ is arbitrarily small, given data from the optimal policy, the confidence intervals used by LinUCB would tightly concentrate around a high regret policy, showing that the confidence intervals used by LinUCB are extremely sensitive to the realizability assumption (See \Cref{app:sensitivity-example}).

To prevent this phenomenon, in the next section we consider an algorithm that constraints the estimate of the outcome model $\hatf$ to be close to $\hatf^*$. This also allows us to derive upper bounds on regret in terms of the deviation of the best in-class model $\hatf^*$ from the true model $f^*$. This characterization is important as it allows us to take into account regret incurred due to model misspecification -- a cost that often assumed away under realizability.

\subsection{Related work on misspecification}

As we have discussed above, bandit algorithms relying on regression oracles are computationally tractable, and when their model is well-specified they often exhibit attractive statistical properties. More recently, there has been interest in developing algorithms that are robust to misspecification. These works differ in how they define and measure misspecification, and how their regret bound degrade as the level of misspecification increases.

\cite{neu2020efficient, zanette2020learning} assume that the absolute deviation between the true reward function and its best linear approximation is at most $\epsilon$ uniformly across contexts and actions. Under this assumption, they develop bandit algorithms whose regret overhead due to misspecification is bounded in terms of this measure of misspecification $\epsilon$. Under the same measure of misspecification, 
\cite{foster2020beyond} provide similar results that hold for any class of models that have an online regression oracle.

This type of uniform bound on model misspecification can be arbitrarily large even in relatively benign examples (see \Cref{app:sensitivity-example}). Concurrent work of \cite{foster2020adapting}  use a different measure of misspecification that allows them to derive tighter regret bounds while relying on online regression oracles. Their measure of misspecification turns out to be very similar to the one we use in this work, however we rely on constrained offline regression oracles instead. Moreover, \cite{foster2020adapting} also adapt to unknown misspecification by relying on master algorithms (see \Cref{sec:discussion}). 

\cite{lattimore2020learning} and \cite{ghosh2017misspecified} also study the related problem of misspecified non-contextual linear bandits. 

\section{Main results}
\label{sec:main_results}

We propose an algorithm that we call $\GeneralizedFalcon$, which is a modification of the ``FAst Least-squares-regression-oracle CONtextual bandits'', or FALCON algorithm described in \cite{simchi2020bypassing}. The main departure from FALCON is that although we do posit some ``tentative'' set $\F$ that could contain the true outcome model, our regret guarantees do not depend on this assumption being satisfied. For simplicity of exposition we will initially assume that $\mathcal{F}$ is a convex subset of a $d$-dimensional linear space \footnote{Consider the class of estimators $\F$ where linear functions estimate rewards for each arm using a total of $d$ parameters. Note that this is a special case of requiring $\F$ to be a convex subset of a $d$-dimensional linear space. Hence, the guarantees in \Cref{thm:linear-theorem} hold for stochastic linear bandits.}, but our results can be extended to more complex classes as we show later.

We will need some additional notation. Let $f^*$ represent the true outcome model, i.e., $f^*(x, a) = \E_{x_t,r_t}[r_t(a)|x_t=x]$ for all $x$ and $a$. Moreover, let $\hatf^*$ denote the best in-class approximation to the true outcome model when data is collected non-adaptively, or
\begin{equation}
  \label{eq:fhatstar}
  \hatf^* :=\arg\min_{f\in\F} \E_{x \sim D_{\Xscript}} \E_{a \sim \Unif(\A)}[(f(x,a) - f^*(x,a))^2],
\end{equation}
where $D_{\Xscript}$ is the distribution of contexts, and $\Unif(\A)$ is a probability distribution that assigns equal probability to every arm. The approximation error between these two functions is denoted as
\begin{align}
  \label{eq:b}
  b := \E_{x\sim D_{\Xscript}}\E_{a\sim\Unif(\A)}[(\hatf^*(x,a)-f^*(x,a))^2].
\end{align}
Naturally, the approximation error \eqref{eq:b} will be zero when realizability holds. And when it doesn't hold, we will show that the algorithm will incur some regret whose upper bound increases with the approximation error. This is what allows us to accurately characterize the cost that we pay when we $\mathcal{F}$ is misspecified (i.e., $f^* \not\in \F$).

 \paragraph{Algorithm:} $\GeneralizedFalcon$ is implemented in increasing epochs (batches) that are indexed by $m$. Each epoch $m$ begins at period $\tau_{m-1}$, we set epoch schedule so that $\tau_0=0$, $\tau_1\geq 4$, and $\tau_{m+1}=2\tau_m$ for any epoch $m\geq 1$. Every epoch $m$ starts out with an estimated reward model $\hatf_m$ obtained at the end of the last batch, with $\hatf_1 \equiv 0$. For a fraction $\epsilon$ of each epoch, called the ``passive'' phase, $\GeneralizedFalcon$ draws actions uniformly at random. For the remaining $1 - \epsilon$ fraction of the epoch, in what we call the ``active'' phase, it acts as a modified version of FALCON.\footnote{More precisely, it acts as a modified version of the FALCON+ algorithm in the same paper, but the distinction is minor enough that we will ignore it for the purposes of naming our method.}

 Our action selection mechanism is the same as FALCON's, so let's briefly review it. At each epoch $m$, given the current reward model estimate $\hatf_m$ and a scaling parameter $\gamma_m > 0$, actions are drawn from the probability distribution described by the following ``action selection kernel'',
 \begin{align}
   \label{eq:action_kernel}
   p_m(a|x):=
   \begin{cases}
   \frac{1}{K+\gamma_m \left(\hatf_m(x, \hat{a}) - \hatf_m(x,a)\right) }&\text{for } a\neq \hat{a}\\
   1 - \sum_{a'\neq \hat{a}} p(a'|x) &\text{for } a=\hat{a}.
   \end{cases}
 \end{align}
 where $\hat{a} = \max_a \hatf_m(a, x)$ is the best predicted action. The assignment rule \eqref{eq:action_kernel} ensures that actions that are predicted to be good according to the current model estimate $\hatf_m$ are given higher probability. The scaling parameter, set to $\gamma_m \simeq \sqrt{K (\tau_{m-1} - \tau_{m-2})/(d \ln( m / \delta))}$ with initial values $\gamma_1 = 1 $, control the degree of exploration during the active phase, with higher values of $\gamma_m$ indicating less exploration. We may sometimes refer to $\gamma_m$ and $\epsilon$ as the active and passive exploration parameters respectively. 
 
The main difference between our method and FALCON is in how we estimate the outcome model $\hatf_{m+1}$ from data collected in the previous epoch $m$. The original algorithm simply uses the estimator that minimizes empirical risk on data collected in the previous time-steps, but as we saw in the example in Section \ref{sec:problem_with_realizability}, when realizability fails the sequence of estimators $\hatf_m$ may not converge to $\hatf^*$. This is due to the fact that the empirical risk minimizer when data is collected adaptively may be very different from the one attained when data is collected non-adaptively, and its performance may not be well understood (See Figure \ref{fig:constraint}). In order to ensure that our estimates converge to $\hatf^*$, our algorithm uses a ``\emph{constrained} regression oracle'' that ensures that the estimated model is always close to the best approximation $\hatf^*$. Let's see how this is done.

Denote the data collected using the passive and active phases of the epoch $m$ by $S_m'$ and $S_m$ respectively. Moreover, let $\F'_m$ denote the subset of functions $f \in \F$ for which the following constraint in satisfied, 
\begin{equation}
\label{eq:linear-constraint-conreg}
    \sum_{(x,a,r(a))\in S_m'} (f(x,a)-r(a))^2 \leq \alpha_m + C_1d\ln(12m^2/\delta) 
\end{equation}
where $\alpha_m := \min_{g \in\F} \sum_{(x,a,r(a))\in S_m'} (g(x,a)-r(a))^2$ is the sum of squared residuals in the model fitted on the data collected in the ``passive'' phase, and $C_1$ is a constant chosen appropriately to ensure that $\hatf^{*}$ also lies in $\F'$ with probability at least $1-\delta/(12m^2)$. \footnote{We pin down the value of this constant in the Appendix}

The estimated model $\hatf_{m+1}$ will be constrained to lie in this set. More specifically, it is the output of the following constrained regression problem:
 \begin{equation}
  \label{eq:constrained_problem}
  \begin{aligned}
   \min_{f\in\F} \quad &\sum_{(x,a,r(a))\in S_m} (f(x,a)-r(a))^2\\
   \textrm{s.t.} \quad & \;\;\; f \in \F'_m.
   \end{aligned}
 \end{equation}
The intuition, again, is that since $\hatf_{m+1} \in \F'_m$ by construction, and since $\hatf^* \in \F'_m$ with high probability, the two will likely remain close. And since $\F'_m$ shrinks over time, $\hatf_{m+1}$ must ultimately converge to $\hatf^*$. Therefore, the convergence issues we saw in our example in Section \ref{sec:problem_with_realizability} cannot happen. This is what allows us to derive regret guarantees even when realizability fails (see Figure \ref{fig:constraint} for an intuitive illustration). The full description and the pseudocode for the general algorithm can be found in the Appendix (\Cref{alg:generalized-falcon}). \footnote{Except for the choice of $\gamma_m$ and the RHS of the constraint \Cref{eq:linear-constraint-conreg}, the algorithm for general $\F$ is the same as the description in this section.}

\begin{auxmulticols}{1}
\begin{figure*}[ht]
  \centering
  \includegraphics[width=.9\textwidth]{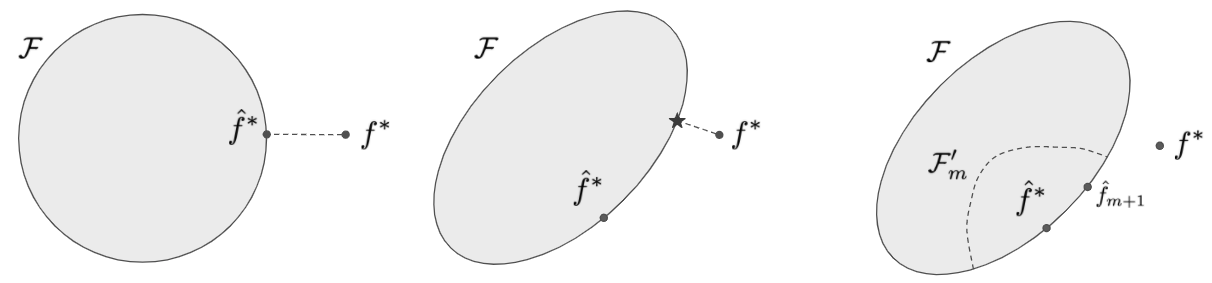}
  \vspace{.3in}
  \caption{Intuition for our method. Left: the function $\hatf^*$ is the best in-class approximation to $f^*$ under non-adaptive data collection. Middle: under a different distribution, the best in-class approximation (starred) may lie very far away from $\hatf^*$, and there are no guarantees on its performance. Right: in our method, we construct a shrinking sequence of sets $\F'_m$ that contain $\hatf^*$ with high probability, and ensure that our model estimates lie in this set.}
    \label{fig:constraint}
\end{figure*}
\end{auxmulticols}

\paragraph{Computational tractability of the constrained regression problem:} Note that $\GeneralizedFalcon$ is very easy to implement given a constrained regression oracle. Hence, for the computational tractability of $\GeneralizedFalcon$, it is sufficient to argue that the constrained regression problem is computationally tractable. When $\F$ is the class of linear reward models, then clearly the constrained regression problem is a convex and can be solved efficiently. In general, when $\F$ is any convex class, we show that the constrained regression problem can be solved efficiently with a weighted regression oracle (see \Cref{app:conreg-tractability}). Hence we can use any of the many existing algorithms for weighted regression as a subroutine to solve the constrained regression problem. While this is one approach to solve the constrained regression problem, in practice directly solving the constrained regression problem may be faster.

Theorem \ref{thm:linear-theorem} provides a high probability regret guarantee for $\GeneralizedFalcon$ when $\F$ is a convex subset of some $d$-dimensional linear space.

\begin{restatable}[Linear case]{theorem}{thmLinear}
\label{thm:linear-theorem} Suppose $\F$ is a convex subset of a $d$-dimensional linear space. With probability at least $1 - \delta$, $\GeneralizedFalcon$ with passive exploration parameter $\epsilon > 0$ attains the following regret guarantee:
\begin{align}
    \label{eq:linear-theorem}
    R_T \leq \ordO\left(\sqrt{KTd\ln\Big(\frac{\ln(T)}{\delta}\Big)} + KT \sqrt{\frac{b}{\sqrt{\epsilon}}} +  \epsilon T \right).
\end{align}
\end{restatable}

The guarantees in \eqref{eq:linear-theorem} consist of three terms. The first term is the regret due to the complexity of the class $\F$, and is the bound guaranteed by realizability based algorithms like FALCON under realizability. The second term can be interpreted as the ``cost of misspecification'', this term depends on the approximation error $b$ and the passive exploration parameter $\epsilon$. Finally, the third term is the regret incurred in the passive phase and depends only on the passive exploration parameter $\epsilon$. 

At first glance, the result in \eqref{eq:linear-theorem} may look rather weak due to the linear dependence in the horizon $T$. However, we contend that any algorithm that that works with a restricted class of policies or reward models, including agnostic algorithms like ILTCB \cite{agarwal2014taming}, will incur some linear regret if these restrictions are violated. In Theorem \ref{thm:linear-theorem} we simply make this issue explicit, as one of our goals is to accurately characterize the bias-variance trade-off in our problem. Our results show that, if the practitioner is willing to spend $\epsilon T$ regret in the passive phase, then in the active phase excess regret due to misspecification will be $\smash{\ordO (KT\sqrt{b/\sqrt{\epsilon}})}$. On the other hand, realizability based approaches do not have any guarantees under general misspecification. 

As a thought experiment, suppose we knew the approximation error $b$ or could make an educated guess about it. In that case we could choose $\epsilon$ as a function of $b$ so as to optimize \eqref{eq:linear-theorem} and obtain the next result.

\begin{restatable}[Linear case with known $b$]{corollary}{coroLinear}
\label{cor:linear-tuning} In the setting of Theorem \ref{thm:linear-theorem}, if the passive exploration parameter is set to $\epsilon = c K^{4/5}b^{2/5}$ for some constant $c > 0$, we have the following bound:
\begin{align}
    \label{eq:linear-tuning}
    R_T \leq \ordO \left(\sqrt{KTd\ln\Big(\frac{\ln(T)}{\delta}\Big)} + K^{4/5} b^{2/5}T \right).
\end{align}
\end{restatable}

This result is interesting because it tells us that if we were able to tune the passive exploration parameter optimally, we get improved regret rates that only depend on the complexity of $\F$ and the approximation error $b$, thus achieving a bias-variance trade-off over the entire horizon $T$. This suggests that tuning $\epsilon$ by estimating $b$ may be a promising direction for future work to get algorithms with better regret guarantees.

\paragraph{Understanding the constrained regression problem:}
Having explained the overall algorithm, let's now understand the constrained regression problem in a bit more detail, so the reader will be able to follow the proof steps in the Appendix. 

At the end of epoch $m$, we have the two kinds of data, that is the data from the passive phase of the epoch and the data from the active phase of the epoch. The data from the passive phase is used to construct $\F'_m$, and contains the best in-class approximation of the true outcome model $\hatf^*$ with high probability (see \Cref{lem:ConReg}). The data from the active phase is used to select a ``good'' estimate within $\F'_m$ which in turn induces a ``good'' action selection kernel. A good action selection kernel has low regret, and ensures that the data generated by this kernel can be used to construct ``good'' estimates in the next epoch. In terms of exploration, there is a trade-off between these two properties as more exploration helps you generate ``good'' data but incurs higher regret. In terms of estimates, both these properties are related because good estimates come from good data. For simplicity let us focus on arguing that the action selections kernels we estimate generate ``good'' data and believe that the active exploration parameter $\gamma_m$ is set optimally. In particular, we say the action selection kernel generates ``good'' data if the reward of the policy induced by $\hatf^*$ can be estimated using the data generated by this kernel. Note that this is trivially ensured when actions are selected uniformly at random, as we did in the first epoch. In later epochs, as the kernel gets less explorative ($\gamma_m$ increases), to ensure this we need the estimator ($\hatf_{m+1}$) that induces this action selection kernel to be close to the best in-class model ($\hatf^*$). More mathematically, as shown in \Cref{lem:boundVopt}, we need the root mean squared difference between $\hatf_{m+1}$ and $\hatf^*$ to shrink at the same rate as $\gamma_m$ increases. This property is guaranteed by the fact that both $\hatf_{m+1}$ and $\hatf^*$ lie in $\F'_m$ with high probability, and by the fact that $\F'_m$ is sufficiently small as we have collected enough data in the passive phase to ensure this (see \Cref{lem:ConReg}). Additionally, this property helps us ensure that our action selection kernels $p_m$ are stable over time, in the sense that if the reward of a policy could be estimated from the data generated by $p_{m+1}$ (in expectation) then the reward of this policy could also be estimated by the data generated by $p_{m}$ (in expectation), see \Cref{lem:boundVall} for a more formal statement. In other words, the set of policies that we implicitly consider do not erratically change over time and only decrease.

\paragraph{General classes of outcome models:} Although for concreteness we have explained our results when $\F$ is a convex subset of a $d$-dimensional linear space, \Cref{thm:linear-theorem} readily extends to more general classes of functions. In particular we can extend \Cref{thm:linear-theorem} whenever $\F$ is a convex and satisfies \Cref{ass:main-assumption}. In terms of the algorithm, except for the choice of $\gamma_m$ and the RHS of the constraint \Cref{eq:linear-constraint-conreg}, $\GeneralizedFalcon$ for general $\F$ is the same as the description in this section. See \Cref{alg:generalized-falcon} in the Appendix for more details. Stating \Cref{ass:main-assumption} can get cumbersome quickly, here we state an informal version of this assumption, followed by \Cref{thm:main-theorem}, and applications of this Theorem to various convex classes $\F$. In what follows $\comp(\F)$ will denote an appropriate measure of complexity, like VC subgraph dimension or entropy. 

\paragraph{Main Assumption:} We now state an informal version of \Cref{ass:main-assumption}. Let $n$ denote the number of data points collected from some distribution. Suppose we have $\rho\in(0,1]$, $\rho'\in[0,\infty)$, and $C>0$. Further suppose for any convex subset $\F'$ of $\F$ and $\zeta \in (0, 1/2)$, with probability $1 - \zeta$, for any $\eta \geq C\ln^{\rho'}(n)\ln(1/\zeta)\comp(\F) / n^{\rho}$, the empirical and true risks of any estimators in $\F'$ are ``close'' in the following sense:
\begin{itemize}
    \item If the population risk of any estimator in $\F'$ is smaller than $\eta$, then its empirical risk is not larger than $3\eta/2$.
    \item If the empirical risk of any estimator in $\F'$ is smaller than $\eta$, then its population risk is not larger than $2\eta$.
\end{itemize}

\begin{restatable}[Main result]{theorem}{thmMain}
\label{thm:main-theorem} Suppose $\F$ is a convex set and suppose \Cref{ass:main-assumption} holds. Then with probability at least $1 - \delta$, $\GeneralizedFalcon$ with passive exploration parameter $\epsilon > 0$ attains the following regret guarantee:
\begin{align}
\label{eq:main-theorem}
\begin{split}
   R_T \leq \ordO\Bigg( &\sqrt{KT^{2-\rho}\ln^{\rho'}(T)\ln(\frac{\ln(T)}{\delta})\comp(\F)}\\ 
   & \;\; + KT \sqrt{\frac{b}{\sqrt{\epsilon^{\rho}}}} + \epsilon T   \Bigg).
\end{split}
\end{align}
\end{restatable}

In \Cref{app:kol-fast-rates}, we provide convenient Lemmas to prove \Cref{ass:main-assumption} for various convex classes $\F$. These Lemmas directly follow from results in \cite{koltchinskii2011oracle}. In fact, \Cref{thm:linear-theorem} is implied by \Cref{thm:main-theorem} and results stated in \Cref{app:kol-fast-rates}. We now go over similar results that follow from \Cref{thm:main-theorem} and \Cref{app:kol-fast-rates}. 


Example 1: Suppose $\F$ is convex and has VC-subgraph dimension $V$. Then with probability $1-\delta$, $\GeneralizedFalcon$ guarantees the following bound on the regret $R_T$: 
$$\ordO\Bigg(\sqrt{KTV\ln\bigg(\frac{T}{V}\bigg)\ln\bigg(\frac{\ln(T)}{\delta}\bigg)} + KT\sqrt{\frac{b}{\sqrt{\epsilon}}} + \epsilon T \Bigg).$$

Example 2: Suppose $\F$ is a convex hull of class with VC-subgraph dimension $V$. Then with probability $1-\delta$, $\GeneralizedFalcon$ guarantees the following bound on the regret $R_T$:
$$\ordO\Bigg(\sqrt{KT^{\frac{2+3V}{2+2V}}V^{\frac{2+V}{2+2V}}\ln\bigg(\frac{\ln(T)}{\delta}\bigg)} + KT\sqrt{\frac{b}{\sqrt{\epsilon}}} + \epsilon T \Bigg).$$

Example 3: Suppose for some $\rho\in(0,1)$, the empirical entropy is bounded by $\ordO(\epsilon^{-2\rho})$ for all empirical distributions. Then with probability $1-\delta$, $\GeneralizedFalcon$ guarantees the following bound on the regret $R_T$:
$$\ordO\Bigg(\sqrt{KT^{\frac{1+2\rho}{1+\rho}}\ln\bigg(\frac{\ln(T)}{\delta}\bigg)} + KT\sqrt{\frac{b}{\sqrt{\epsilon^{1/(1+\rho)}}}} + \epsilon T \Bigg). $$



\section{Discussion}
\label{sec:discussion}

This paper's contribution is twofold. First, to illustrate how algorithms that rely on realizability may incur unexpected regret when this assumption is violated. We saw in Section \ref{sec:problem_with_realizability} that one can construct examples where regret is large even in relatively benign settings. Second, to propose a flexible family of computationally tractable algorithm that are less sensitive to realizability. Our analysis  in Section~\ref{sec:main_results} characterizes the behavior of regret under misspecification and gives us insight into the bias-variance trade-off in contextual bandits. 

In terms of algorithm design, our proposed algorithm $\GeneralizedFalcon$ inherits the computational elegance of realizability based approaches like FALCON. In particular, a single estimator gives you an implicit distribution over policies via the action selection kernel and bypasses the need to explicitly construct a distribution over policies. Our key insight is that by using a constrained regression estimator, we can make this approach robust to misspecification at the expense of some additional regret in the passive phase. 

We believe this work represents an important step towards the development of contextual bandit algorithms that are robust to misspecification. Natural extensions include the following.

\paragraph{Adapting to misspecification} The performance of $\GeneralizedFalcon$ depends on the input parameter ($\epsilon$). One natural way to address this deficiency may be to initialize multiple base algorithms with different choices of $\epsilon$ and use a master algorithm \cite{agarwal2017corralling} to choose the best performing base algorithm. In fact, the recent work of \cite{foster2020adapting} take this approach to adapt to unknown misspecification. The idea of using a master algorithm to adapt to unknown misspecification also appears in \cite{pacchiano2020model}, they use the algorithm in \cite{zanette2020learning} as a base algorithm to adapt to an unknown uniform misspecification error for linear contextual bandits. The final drawback is that we use naive uniform sampling for the passive phase. One may be able to achieve a tighter regret bound by using a more sophisticated exploration scheme for the passive phase.

\paragraph{More general classes} $\GeneralizedFalcon$ requires the model class $\F$ to be convex. We use the convexity of $\F$ in several ways. When $\mathcal{F}$ is convex and has finite VC-dimension we get that $\rho=1$ in \Cref{ass:main-assumption}. Convexity of $\mathcal{F}$ also allows us to solve the constrained regression oracle using only an offline weighted regression oracle (Section D). More importantly, convexity of $\mathcal{F}$ helps ensure that there is a unique best in-class estimator ($\hat{f}^*$) up to evaluation on a non-zero measure set, which in turn helps ensure that the active policy ($\pi_{\hat{f}_m}$) converges to our target policy ($\pi_{\hatf^*}$); see Lemmas \ref{lem:ConReg}, \ref{lem:boundVopt}, and \ref{lem:boundVall}. As many online regression algorithms rely on convexity, one may expect this drawback to implicitly hold for algorithms that rely on online regression oracles such as \cite{foster2020beyond, foster2020adapting}.

To close, we note that these results hint at the possibility of exploiting the bias-variance trade-off in tractable contextual bandits to perform good model selection. This would be an interesting direction for future work.

\section{Acknowledgments}

We are grateful for the generous financial support provided by the Sloan Foundation, Schmidt Futures and the Office of Naval Research grant N00014-19-1-2468. SKK acknowledges generous support from the Dantzig-Lieberman Operations Research Fellowship.

\bibliography{ref}
\bibliographystyle{apalike}

\clearpage
\appendix
\onecolumn
%
%

\section{Detailed setup}
\label{app:detailed-setup}

In \Cref{app:proofs} we will prove the claims in the body of the paper. This requires us to establish some additional notation, which we do in \Cref{app:preliminaries}. Most of these symbols and definitions were used in the original FALCON paper \cite{simchi2020bypassing}. The results in \Cref{app:kol-fast-rates} use notation and definitions from \cite{koltchinskii2011oracle} and are stated within \Cref{app:kol-fast-rates}. \Cref{app:main-assumption} states the main assumption used in \Cref{thm:main-theorem}, and \Cref{app:algorithm} describes the general version of $\GeneralizedFalcon$.

\subsection{Preliminaries}
\label{app:preliminaries}

To start, let $\Gamma_{t}$ denote the set of observed data points up to and including time $t$. That is 
\begin{align}
    \label{eq:history}
    \Gamma_{t} := \{(x_s, a_s, r_s(a_s))\}_{s=1}^{t}
\end{align}

Recalling the text, an ``action selection kernel'' $p$ gives us the probability $p(a|x)$ of selecting an arm~$a$ given a context $x$, and a ``policy'' is a deterministic mapping from contexts to actions. Let $\Psi = \A^{\Xscript}$ denote the universal policy space containing all possible policies. Following Lemma 3 in \cite{simchi2020bypassing}, given any action selection kernel $p$ we can construct a unique product probability measure on $\Psi$, given by:
\begin{equation}
    \label{eq:q-product}
    Q_p(\pi) := \prod_{x\in\Xscript} p(\pi(x)|x),
\end{equation}
and it satisfies the following property
\begin{align}
    \label{eq:connecting-p-and-Q}
  p(a|x) = \sum_{\pi\in\Psi} \I\{\pi(x)=a\}Q_p(\pi).
\end{align}
Property \eqref{eq:connecting-p-and-Q} establishes a duality between action selection kernels, which are used in practice in the algorithm implementation, and the probability distribution \eqref{eq:q-product}, which is a theoretical object that can be used to simplify the proofs below. For short-hand, we let $Q_m \equiv Q_{p_m}$ denote the product probability measure on $\Psi$ induced by the action selection kernel $p_m$ defined in \eqref{eq:action_kernel}. 

Now, for any action selection kernel $p$ and any policy $\pi$, we let $V(p,\pi)$ denote the expected inverse probability. 
\begin{equation}
    \label{eq:decisional-divergence}
    V(p,\pi):=\E_{x\sim D_{\Xscript}}\bigg[\frac{1}{p(\pi(x)|x)}\bigg]    
\end{equation}
One can interpret \eqref{eq:decisional-divergence} as a measure of average divergence between $p(\cdot|x)$ and $\pi(x)$. \cite{simchi2020bypassing} refer to this as the decisional divergence between the randomized policy $Q_p$ and deterministic policy $\pi$.

Given an outcome model $f$ and policy $\pi$, we can define the expected instantaneous reward of the policy $\pi$ with respect to the model $f$ as
\begin{equation}
    \label{eq:instantaneous-reward}
    R_f(\pi) := \E_{x \sim D_\Xscript}[f(x, \pi(x))].
\end{equation}
When there is no possibility of confusion, we will write $R(\pi)$ to mean $R_{f^*}(\pi)$, the reward with respect to the true model $f^*$.. The policy $\pi_f$ induced by the model $f$ is defined by setting $\pi_f(x) := \arg\max_a f(x, a)$ for every $x$. Note that this policy has the highest instantaneous reward with respect to the model $f$, that is $\pi_f=\arg\max_{\pi\in\Psi} R_f(\pi)$.  We can also define the expected instantaneous regret with respect to the outcome model $f$ as
\begin{equation}
    \label{eq:instantaneous-regret}
    \Reg_f(\pi) := \E_{x \sim D_\Xscript}[f(x,\pi_f(x)) - f(x, \pi(x))].
\end{equation}
When there is no possibility of confusion, we will write $\Reg(\pi)$ to mean $\Reg_{f^*}(\pi)$, the regret with respect to the true model $f^*$.

Recall that we define $\hatf^*$ as the best in-class approximation to the true outcome model when actions are sampled uniformly at random. Also recall that we define $b$ as the approximation error or mean squared difference between $\hatf^*$ and $f^*$ when actions are sampled uniformly at random. We now define $B$ to be the largest mean squared difference between $\hatf^*$ and $f^*$ under any action selection kernel. That is, \footnote{\Cref{lem:boundB} bounds $B$ with $Kb$.}
\begin{equation}
    \label{eq:capb}
    B:=\max_p \E_{x\sim D_{\Xscript}}\E_{a\sim p(\cdot|x)} [(\hatf^*(x,a)-f^*(x,a))^2] = \E_{x\sim D_{\Xscript}}[\max_a(\hatf^*(x,a)-f^*(x,a))^2].     
\end{equation}

\subsection{Main assumption}
\label{app:main-assumption}

\begin{assumption}
\label{ass:main-assumption}

Suppose that our outcome model $\F$ satisfies the following property. There exists constants  $C>0$, $\rho\in(0,1]$, $\rho'\in[0,\infty)$ such that for any action selection kernel $p$,  any convex subset $\F' \subset \F$, any natural number $n$, any $\zeta \in (0, 1)$, and any $\eta > C\ln^{\rho'}(n)\ln(1/\zeta)\comp(\F) / n^{\rho}$, the following holds with probability at least $1-\zeta$:
\begin{equation}
    \label{eq:main-assumption}
    \F'(\eta, p) \subseteq \widehat{\F'}(3\eta/2, \Stilde)
        \qquad \text{and} \qquad
    \widehat{\F'}(\eta, \Stilde) \subseteq  \F'(2\eta, p),
\end{equation}
where the $\eta$-minimal set is defined as
\begin{equation}
    \label{eq:minimal-set-population}
    \mathcal{F'}(\eta, p) := \left\{ f \in \F' \, \bigg| \, 
             \E_{(x_i,r_i)\sim D}\E_{a_i\sim p(\cdot|x)}[ (f(x_i, a_i) - r_i(a_i))^2 ]
            \leq 
        \min_{\tilde{f} \in \F'}
            \E_{(x_i,r_i)\sim D}\E_{a_i\sim p(\cdot|x)}[ (\tilde{f}(x_i, a_i) - r_i(a_i))^2 ]
            + \eta
    \right\},
\end{equation}
and the \emph{empirical} $\eta$-\emph{minimal set} is defined as 
\begin{equation}
    \label{eq:minimal-set-empirical}
    \widehat{\F'}(\eta, \Stilde) := \left\{ f \in \F \, \bigg| \, 
        \frac{1}{n} \sum_{i=1}^{n} (f(x_i, a_i) - r_i(a_i))^2 
            \leq 
        \min_{\tilde{f} \in \F'} 
            \frac{1}{n} \sum_{i=1}^{n} (\tilde{f}(x_i, a_i) - r_i(a_i))^2    
            + \eta
    \right\}.
\end{equation}
and where the data $\Stilde \equiv (x_i, a_i, r_i(a_i))_{i=1}^{n}$ are drawn independently and identically from $x_i \sim \mathcal{D}_{\Xscript}$, $a_i | x_i \sim p(\cdot | x_i)$ and $r_i \sim \mathcal{D}_{r_i|x_i, a_i}$, and the expectations are taken with respect to these distributions.
\end{assumption}

\subsection{Algorithm}
\label{app:algorithm}


The general version of our algorithm for general classes of outcome models $\F$ requires three modifications. Note the constants $C$, $\rho$, and $\rho'$ mentioned below are rate terms from \Cref{ass:main-assumption}, $C_3 := 1/(4C_5)$  (see \Cref{lem:reg-est-accuracy}), and $C_5 := 2C\times 4^{\rho} \times (2+\ln(12))$ (see \Cref{lem:ConReg}).

First, the epoch schedule needs to satisfy $\tau_0 = 0$, $\tau_1 \geq 4$ and for subsequent epochs we set $\tau_{m+1}= 2\tau_m$.

Second, the parameter $\gamma_t$ is set to $\gamma_1 = 1$ and 
\begin{equation}
    \label{eq:gamma}
    \gamma_m = \sqrt{\frac{C_3K(\tau_{m-1}-\tau_{m-2})^{\rho}}{\ln^{\rho'}(\tau_{m-1}-\tau_{m-2})\ln((m-1)/\delta)\comp(\F)}}.
\end{equation}

Third and finally, the constraint set $\F_m'$ consists of the set of outcome models $f \in \F$ such that
\begin{align}
    \label{eq:generalized_constraint}
    \F'_m := \left\{ f \in \F \,\, \bigg| \,\,
        \frac{1}{|S_m'|}\sum_{S_m'}
            (f_{m+1}(x,a)-r(a))^2 \leq 
            \alpha_m + 
            \frac{C_1\ln^{\rho'}(|S_m'|)\ln(1/\delta')\comp(\F)}{|S_m'|^{\rho}} 
\right\},
\end{align}
where $\alpha_m := \frac{1}{|S_m'|}\min_{g\in\F} \sum_{S_m'} (g(x,a)-r(a))^2$, $\delta' = \delta/(12m^2)$, and $C_1 = 3C/2$ (see \Cref{lem:ConReg}).

\begin{algorithm}[H]
  \caption{$\GeneralizedFalcon$}
  \label{alg:generalized-falcon}
  \textbf{input:} epoch schedule $\tau_1\geq 4$, confidence parameter $\delta$, and forced exploration parameter $\epsilon$.
  \begin{algorithmic}[1] 
  \State Set $\tau_0 = 0$, and $\tau_{m+1} = 2\tau_m$ for all $m\geq 1$.
  \State Let $\hatf_1 \equiv 0$.
  \For{epoch $m=1,2,\dots$}
    \State Let $\gamma_m=\sqrt{\frac{C_3K(\tau_{m-1}-\tau_{m-2})^{\rho}}{\ln^{\rho'}(\tau_{m-1}-\tau_{m-2})\ln((m-1)/\delta)\comp(\F)}}$ (for epoch 1, $\gamma_1=1$).
    \For{round $t=\tau_{m-1}+1,\dots, \tau_{m} - \ceil*{ \epsilon(\tau_m - \tau_{m-1})}$ }
      \State Observe context $x_t$, let $\hata_t=\arg\max_{a\in\A}\hatf_m(x_t,a)$, and define:
        \begin{align*}
          p_t(a):=
          \begin{cases}
              \frac{1}{K+\gamma_m(\hatf_m(x_t,\hata_t) - \hatf_m(x_t,a))}, &\text{for all $a\neq\hata_t$}\\
              1 - \sum_{a'\neq\hata_t} p(a'|x),                             &\text{for $a=\hata_t$}
              \end{cases}
        \end{align*}
      \State Sample $a_t \sim p_t(\cdot)$ and observe $r_t(a_t)$.
    \EndFor
    \For{round $t=\tau_{m} - \ceil*{ \epsilon(\tau_m - \tau_{m-1})}+1,\dots,\tau_{m}$}
        \State Observe context $x_t$, sample $a_t$ uniformly at random from $\A$, and observe $r_t(a_t)$.
    \EndFor
    \State Let:
    \begin{align*}
    	S_m &=\{(x_t,a_t,r_t(a_t))\}_{t=\tau_{m-1}+1}^{\tau_{m} - \ceil*{ \epsilon(\tau_m - \tau_{m-1})}} \\
      S_m'&=\{(x_t,a_t,r_t(a_t))\}_{t=\tau_{m} - \ceil*{ \epsilon(\tau_m - \tau_{m-1})} + 1}^{\tau_m}.
    \end{align*}
    \State Compute $\hatf_{m+1}$ by solving
        
     \begin{equation}
      \label{eq:generalized_constrained_problem}
      \begin{aligned}
      \min_{f\in\F} \quad &\sum_{(x,a,r(a))\in S_m} (f(x,a)-r(a))^2 \\
      \textrm{s.t.} \quad & \;\;\; f \in \F'_m.
      \end{aligned}
     \end{equation}
    where $\F'_m$ is defined as in \eqref{eq:generalized_constraint}.
  \EndFor
  \end{algorithmic}
\end{algorithm}

\section{Proofs}
\label{app:proofs}

The goal of this section is to present our proof of Theorem \ref{thm:main-theorem}. Section \ref{app:overview-of-results} gives a brief overview of the argument. Section \ref{app:main-assumption} restates the main assumption. Sections \ref{app:bounds-on-best-predictor}-\ref{app:bounding-true-regret} prove auxiliary Lemmas, and finally Section \ref{app:proof-of-main-theorem} concludes with a proof of the theorem. A small, more technical, portion of the argument is deferred to Section \ref{app:kol-fast-rates}.

\subsection{Overview of the proof for \Cref{thm:main-theorem}}
\label{app:overview-of-results}

For convenience, here is an informal, abridged version of the argument used in the proofs. We hope the reader will find it useful to navigate the results that follow.

\begin{itemize}
    \item First of all, during the passive phase we always incur $\epsilon T$ regret. For the remainder, let's consider the regret incurred during periods occurring in the active phase of each epoch.
    
    \item The cumulative regret incurred across the active phases  will be close to the sum of its conditional expectations at each period,
    \begin{equation*}
        \sum_{t \in \mathcal{T}_{\text{active}}} r_t(\pi^*(x))-r_t(a_t) \approx 
        \sum_{t \in \mathcal{T}_{\text{active}}}
        \E_{x_t,r_t,a_t}[r_t(\pi^*(x))-r_t(a_t)|\Gamma_{m(t)-1}] \qquad \text{w.h.p.},
    \end{equation*}
    so we only need to bound these conditional expectations.
    
    \item By \Cref{lem:conditional-reward}, the conditional expectation of instantaneous regret at period $t$ in the active phase of epoch $m$ can be rewritten in terms of the probability measure $Q_m$ over policies,
    \begin{equation*}
        \E_{x_t,r_t,a_t}[r_t(\pi^*(x))-r_t(a_t)|\Gamma_{m(t)-1}] =  \sum_{\pi\in\Psi}Q_m(\pi)\Reg_{f^*}(\pi).
    \end{equation*}
    
    \item By design, our method will produce a sequence of actions such that the \emph{estimated} regret $\Reg_{\hatf_m}(\pi)$ is small for the policies that receive high probability under $Q_m$ (see \Cref{lem:QmRegEst}). In order to show that the \emph{expected} regret $\Reg_{f^*}(\pi)$ is also small, we need to show that the two are ``close'', at least for policies that receive high probability under $Q_m$.
    
    \item Naturally the difference between expected and estimated regret depends on how closely the sequence $\hatf_m$ approximates $f^*$. In \Cref{lem:ConReg}, we characterize this approximation as a function of two objects: the expected distance between $\hatf_m$ and the best in-class approximation $\hatf^*$, and the distance between $\hatf^*$ and the true model $f^*$. The former decreases at a rate characterized by $1/\gamma_m$ due to properties of our constrained regression problem. The latter is upper bounded by $B$. Therefore,
    \begin{equation*}
        \begin{aligned}
         \E_{x\sim D_{\Xscript}}\E_{a\sim \Unif(\A)}[(\hatf_{m+1}(x,a)-\hatf^*(x,a))^2] &\lesssim \frac{1}{\epsilon^{\rho} \gamma_m}  \\
    \E_{x\sim D_{\Xscript}}\E_{a\sim p_m(\cdot|x)}[(\hatf_{m+1}(x,a)-f^*(x,a))^2] &\lesssim B +  \frac{1}{\gamma_m}.
        \end{aligned}
    \end{equation*}
    
    \item In \Cref{lem:reg-est-accuracy}, we extend these results to bound on the approximation error for any policy $\pi$, 
       \begin{equation*}
        \left| \E_{x \sim \mathcal{D}_{\Xscript}}[\hatf_{m+1}(x, \pi(x)) - f^*(x, \pi(x))] \right| 
        \lesssim
        \sqrt{V(p_m,\pi)} \left( 
            \sqrt{B} + \frac{\sqrt{K}}{\gamma_{m}} 
        \right).
    \end{equation*}
    
    \item Lemmas \ref{lem:boundVopt} and \ref{lem:boundVall} characterize the behavior of the object $V(p_m, \pi)$. In \Cref{lem:policy-reg-bound} we use these results to show that estimated and expected regret satisfy the following relation, which formalized the notion of ``closeness'' between the two:
    \begin{equation*}
        \begin{aligned}
            \Reg_{f^*}(\pi) 
                \lesssim 
                \Reg_{\hatf_{m}}(\pi) + \frac{K}{\gamma_m} +     \sqrt{\frac{KB}{\sqrt{\epsilon^{\rho}}}} +     \sqrt{V(p_m,\pi)B} \\
            \Reg_{\hatf_{m}}(\pi) 
                \lesssim
                \Reg(\pi) + \frac{K}{\gamma_m} +     \sqrt{\frac{KB}{\sqrt{\epsilon^{\rho}}}} + \sqrt{V(p_m,\pi)B}.
        \end{aligned}
    \end{equation*}
    
    \item \Cref{lem:QmRegTrue} concludes that the average expected regret suffered during any point in the active phase is bounded by
    \begin{equation*}
        \sum_{\pi \in \Psi} Q_m(\pi)\Reg(\pi)
            \lesssim \frac{K}{\gamma_m} + 
            \sqrt{\frac{KB}{\sqrt{\epsilon^{\rho}}}}.
    \end{equation*}

    \item In subsection \ref{app:proof-of-main-theorem} we put all of these results together to prove \Cref{thm:main-theorem}.
\end{itemize}






\subsection{Bounds on best predictor}
\label{app:bounds-on-best-predictor}

In this subsection we provide basic bounds on terms involving the best predictor. We start by bounding the empirical mean square error between the best predictor ($\hatf^*$) and the true model ($f^*$) under any action selection kernel, see \Cref{lem:boundB}. We then use this to bound the regret of the policy induced by the best predictor ($\pi_{\hatf^*}$), see \Cref{lem:reg-of-best-pred}. Hence indicating that this policy is a reasonable policy to try to converge to.

\begin{restatable}[Bounding $B$]{lemma}{lemboundB}
\label{lem:boundB}
For any action selection kernel $p$, we then have that:
$$ \E_{x\sim D_{\Xscript}}\E_{a\sim p(\cdot|x)}[(\hatf^*(x,a)-f^*(x,a))^2] \leq B \leq Kb.  $$
\end{restatable}
\begin{proof}
We get the first inequality from the definition of $B$:
\begin{align*}
    \E_{x\sim D_{\Xscript}}\E_{a\sim p(\cdot|x)}[(\hatf^*(x,a)-f^*(x,a))^2] \leq  \max_{p'} \E_{x\sim D_{\Xscript}}\E_{a\sim p'(\cdot|x)} [(\hatf^*(x,a)-f^*(x,a))^2] = B.
\end{align*}
For any context $x\in\Xscript$, note that:
\begin{align*}
    \E_{a\sim p'(\cdot|x)}[(\hatf^*(x,a)-f^*(x,a))^2] \leq & \sum_{a\in\A} (\hatf^*(x,a)-f^*(x,a))^2.
\end{align*}
Now, taking expectations on both sides gives us the second inequality of \Cref{lem:boundB}:
\begin{align*}
    B \leq \sum_{a\in\A} \E_{x\sim D_{\Xscript}} [(\hatf^*(x,a)-f^*(x,a))^2] = Kb.
\end{align*}
\end{proof}

\begin{lemma}[Regret of the policy induced by the best predictor]
\label{lem:reg-of-best-pred}
  We have the following bound on the regret of $\pi_{\hatf^*}$:
  \begin{align*}
      \Reg(\pi_{\hatf^*}) := R(\pi_{f^*}) - R(\pi_{\hatf^*}) \leq 2\sqrt{B}.
  \end{align*}
\end{lemma}
\begin{proof}
  Note that, for any policy $\pi$, we have:
  \begin{align*}
      |R_{\hatf^*}(\pi) - R(\pi)|^2 = |\E_{x\sim D_{\Xscript}}[\hatf^*(x,\pi(x)) - f^*(x,\pi(x))]|^2 \leq \E_{x\sim D_{\Xscript}}[ (\hatf^*(x,\pi(x)) - f^*(x,\pi(x)))^2 ] \leq B.
  \end{align*}
  Where the last inequality follows from \Cref{lem:boundB}. Hence for any policy $\pi$, we have that:
  \begin{align*}
      R(\pi_{\hatf^*}) \geq R_{\hatf^*}(\pi_{\hatf^*}) - \sqrt{B} \geq R_{\hatf^*}(\pi) - \sqrt{B} \geq R(\pi) - 2\sqrt{B}.
  \end{align*}
  In particular, this implies that $\Reg(\pi_{\hatf^*}) := R(\pi_{f^*}) - R(\pi_{\hatf^*}) \leq 2\sqrt{B}$.
\end{proof}

\subsection{Properties of the action selection kernel}
\label{app:properties-action-selection-kernel}

In this subsection, we explore properties of the algorithm that directly follow from the definitions in \Cref{app:detailed-setup} and from the form of the action kernel used in the active phase of \GeneralizedFalcon. For this reason, all the properties stated here hold true for the Falcon algorithm as well. Except for \Cref{lem:QmRootV} and the lower bound in \Cref{lem:boundV}, all Lemmas in this subsection have been proved for Falcon and can be found in \cite{simchi2020bypassing}. We state and prove these Lemmas that we use for completeness and to show that they hold for \GeneralizedFalcon as well. We start with \Cref{lem:conditional-reward} which shows that the expected instantanious regret is equal to the regret of the randomized policy $Q_m$.

\begin{restatable}[Conditional expected reward]{lemma}{lemconditionalreward}
\label{lem:conditional-reward}
For any epoch $m\geq 1$ and time-step $t \geq 1$ in the active phase of epoch $m$, we have:
$$ \E_{x_t,r_t,a_t}[r_t(\pi^*(x))-r_t(a_t)|\Gamma_{t-1}] =  \sum_{\pi\in\Psi}Q_m(\pi)\Reg(\pi). $$
\end{restatable}
\begin{proof}
Consider any epoch $m \geq 1$ and time-step $t
\geq 1$ in the active phase of epoch $m$, then from \Cref{eq:connecting-p-and-Q} we have:
\begin{align*}
\begin{split}
  &\E_{x_t,r_t,a_t}[r_t(\pi^*(x))-r_t(a_t)|\Gamma_{t-1}]\\
  & = \E_{x\sim D_{\Xscript}, a \sim p_m(\cdot|x)}[f^*(x,\pi^*))-f^*(x,a)]\\
  & = \E_{x\sim D_{\Xscript}}\Bigg[\sum_{a\in\A}p_m(a|x)(f^*(x,\pi^*))-f^*(x,a))\Bigg]\\
  & = \E_{x\sim D_{\Xscript}}\Bigg[\sum_{a\in\A}\sum_{\pi\in\Psi}\I(\pi(x)=a)Q_m(\pi)(f^*(x,\pi^*))-f^*(x,a))\Bigg]\\
  & = \sum_{\pi\in\Psi}Q_m(\pi)\E_{x\sim D_{\Xscript}}\Bigg[(f^*(x,\pi^*))-f^*(x,\pi(x)))\Bigg]\\
  & =  \sum_{\pi\in\Psi}Q_m(\pi)\Reg(\pi).
\end{split}
\end{align*}
\end{proof}

Lemma~\ref{lem:QmRegEst} states a key bound on the estimated regret of the randomized policy $Q_m$.
\begin{restatable}[Action selection kernel has low estimated regret]{lemma}{lemQmRegEst}
\label{lem:QmRegEst}
For any epoch $m\geq 1$, we have:
$$ \sum_{\pi\in\Psi} Q_m(\pi)\Reg_{\hatf_m}(\pi) \leq \frac{K}{\gamma_m}. $$
\end{restatable}
\begin{proof}
Note that:
\begin{align*}
    & \sum_{\pi\in\Psi} Q_m(\pi)\Reg_{\hatf_m}(\pi) = \sum_{\pi\in\Psi} Q_m(\pi)\E_{x\sim D_{\Xscript}}\Big[ \hatf_m(x,\pi_{\hatf_m}(x)) - \hatf_m(x,\pi(x)) \Big]\\
    & = \E_{x\sim D_{\Xscript}}\Big[ \sum_{\pi\in\Psi} Q_m(\pi)\Big(\hatf_m(x,\pi_{\hatf_m}(x)) - \hatf_m(x,\pi(x))\Big) \Big]\\
    & = \E_{x\sim D_{\Xscript}}\Big[ \sum_{a\in\A} \sum_{\pi\in\Psi} \I(\pi(x)=a) Q_m(\pi)\Big(\hatf_m(x,\pi_{\hatf_m}(x)) - \hatf_m(x,a)\Big) \Big]\\
    & = \E_{x\sim D_{\Xscript}}\Big[ \sum_{a\in\A} p_m(a|x) \Big(\hatf_m(x,\pi_{\hatf_m}(x)) - \hatf_m(x,a)\Big) \Big]\\
    & = \E_{x\sim D_{\Xscript}}\Bigg[ \sum_{a\in\A} \frac{\Big(\hatf_m(x,\pi_{\hatf_m}(x)) - \hatf_m(x,a)\Big)}{K+\gamma_m\Big(\hatf_m(x,\pi_{\hatf_m}(x)) - \hatf_m(x,a)\Big)} \Bigg] \leq \frac{K}{\gamma_m}.
\end{align*}
\end{proof}

Lemma \ref{lem:QmRootV} is a direct concequence of Jensen's inequality and helps us in the derivation of \Cref{lem:QmRegTrue}, which bounds the true regret of the randomized policy $Q_m$.
\begin{restatable}[An implication of inherent duality between $p_m$ and $Q_m$]{lemma}{lemQmRootV}
\label{lem:QmRootV}
For any epoch $m \geq 1$, we have:
$$ \sum_{\pi\in\Psi} Q_m(\pi)\sqrt{V(p_m,\pi)} \leq \sqrt{K}.$$
\end{restatable}
\begin{proof}
Note that:
\begin{align*}
    & \sum_{\pi\in\Psi} Q_m(\pi)\sqrt{V(p_m,\pi)} \leq \sqrt{\sum_{\pi\in\Psi} Q_m(\pi)V(p_m,\pi)} = \sqrt{\sum_{\pi\in\Psi} Q_m(\pi)\E_{x\sim D_{\Xscript}}\bigg[ \frac{1}{p_m(\pi(x)|x)} \bigg]}\\
    & = \sqrt{\E_{x\sim D_{\Xscript}}\bigg[ \sum_{\pi\in\Psi} Q_m(\pi) \sum_{a\in\A} \frac{\I(\pi(x)=a)}{p_m(a|x)} \bigg]} = \sqrt{\E_{x\sim D_{\Xscript}}\bigg[ \sum_{a\in\A} \frac{\sum_{\pi\in\Psi} \I(\pi(x)=a)Q_m(\pi)}{p_m(a|x)} \bigg]}\\
    &= \sqrt{\E_{x\sim D_{\Xscript}}\bigg[ \sum_{a\in\A} \frac{p_m(a|x)}{p_m(a|x)} \bigg]} = \sqrt{K}.
\end{align*}
Where the first inequality is an application of Jensen's inequality, and the other equalities are straight forward.
\end{proof}

For any policy $\pi$, \Cref{lem:boundV} provides key bounds on $V(p_m,\pi)$. These bounds help us understand the average divergence between the action distribution $p_m(\cdot|x)$ and action selected by the policy $\pi(x)$. 

\begin{restatable}[Bounds on expected inverse probability]{lemma}{lemboundV}
\label{lem:boundV}
For all policies $\pi \in \Psi$ and epochs $m \geq 1$, we have:
\begin{align*}
    &\gamma_m \E_{x\sim D_{\Xscript}}\Big[\big(\hatf_{m}(x,\pi_{\hatf_{m}}(x))- \hatf_{m}(x,\pi(x))\big) \Big]
    \leq V(p_m,\pi) \leq K + \gamma_m \E_{x\sim D_{\Xscript}}\Big[\big(\hatf_{m}(x,\pi_{\hatf_{m}}(x))- \hatf_{m}(x,\pi(x))\big) \Big]
\end{align*}
\end{restatable}
\begin{proof}
	Consider any policy $\pi\in\Psi$ and epoch $m\geq 1$. For any context $x\in\Xscript$ and action $a\in \A \setminus \{\pi_{\hatf_{m}}(x)\}$, from our choice for $p_m$, we get:
	$$
	\frac{1}{p_m(a|x)}= K+\gamma_m(\hatf_m(x,\pi_{\hatf_{m}}(x)) - \hatf_m(x,a)).
	$$
	For the action $a= \pi_{\hatf_{m}}(x)$, we have:
	\begin{align*}
	0 = \gamma_m\Big[\big(\hatf_{m}(x,\pi_{\hatf_{m}}(x))- \hatf_{m}(x,a)\big) \Big] \leq \frac{1}{p_m(a|x)} = \frac{1}{1- \sum_{a'\neq a}\frac{1}{K+\gamma_m \big(\hatf_{m}(x,\pi_{\hatf_{m}}(x))- \hatf_{m}(x,a')\big) }} \leq K
	\end{align*}
	In particular, putting the above inequality together, we get:
	\begin{align*}
	\gamma_m\Big[\big(\hatf_{m}(x,\pi_{\hatf_{m}}(x))- \hatf_{m}(x,\pi(x))\big) \Big] \leq \frac{1}{p_m(\pi(x)|x)} \leq K + \gamma_m \Big[\big(\hatf_{m}(x,\pi_{\hatf_{m}}(x))- \hatf_{m}(x,\pi(x))\big) \Big].
	\end{align*}
	The Lemma now follows by taking expectation over $x\sim D_{\Xscript}$.
\end{proof}

\subsection{Constrained regression oracle guarantees}

\begin{restatable}[Guarantees on the constrained regression oracle]{lemma}{lemConReg}
\label{lem:ConReg}
Suppose \Cref{ass:main-assumption} holds and suppose $\epsilon<0.5$. Then there exists positive constants $C_4$ and $C_5$ such that with probability at least $1-\delta/2$, the following holds for all epoch $m \geq 1$:
\begin{align*}
\E_{x\sim D_{\Xscript}}\E_{a\sim \Unif(\A)}[(\hatf_{m+1}(x,a)-\hatf^*(x,a))^2] \leq \frac{C_4\ln^{\rho'}(\tau_m-\tau_{m-1})\ln(m/\delta)\comp(\F)}{(\epsilon(\tau_m-\tau_{m-1}))^{\rho}}.  \\
\E_{x\sim D_{\Xscript}}\E_{a\sim p_m(\cdot|x)}[(\hatf_{m+1}(x,a)-f^*(x,a))^2] \leq B +  \frac{C_5\ln^{\rho'}(\tau_m-\tau_{m-1})\ln(m/\delta)\comp(\F)}{(\tau_m-\tau_{m-1})^{\rho}}.
\end{align*}
\end{restatable}
\begin{proof}
	Let $\F'$ denote the set of estimators in the constraint set  at the end of epoch $m$. Let $\delta' = \delta/(12m^2)$. Since $\hatf_{m+1}\in\F'$, we have:
	\begin{align*}
	 \frac{1}{|S_m'|}\sum_{(x,a,r(a))\in S_m'} (\hatf_{m+1}(x,a)-r(a))^2 - \frac{1}{|S_m'|}\min_{g\in\F} \sum_{(x,a,r(a))\in S_m'} (g(x,a)-r(a))^2 \\
   \leq  \frac{C_1\ln^{\rho'}(|S_m'|)\ln(1/\delta')\comp(\F)}{|S_m'|^{\rho}}.
	\end{align*}
	The above inequality bounds the empirical excess risk for $\hatf_{m+1}$ with respect to the empirical data $S_m'$ and the set of estimators in $\F$. Now note that $S_m'$ is generated by sampling actions uniformly at random, and note that $\F$ is a convex set. Hence from \Cref{ass:main-assumption}, we get that for some universal constant $L_1=2\max\{C,C_1\}$ \footnote{Where $C$ is the constant from \Cref{ass:main-assumption}.}, with probability at least $1-\delta'$, we have:
	\begin{align}
	\label{ineq:uniform-population-bound-on-estimate-of-epoch-m}
	\begin{split}
	\E_{(x,r)\sim D}\E_{a\sim\Unif(\A)}[(\hatf_{m+1}(x,a)-r(a))^2] - \E_{(x,r)\sim D}\E_{a\sim\Unif(\A)}[(\hatf^*(x,a)-r(a))^2]\\
  \leq  \frac{L_1\ln^{\rho'}(|S_m'|)\ln(1/\delta')\comp(\F)}{|S_m'|^{\rho}}.
	\end{split}
	\end{align}
	Since $\F$ is a convex class of functions, Lemma 5.1 in \cite{koltchinskii2011oracle} gives us that:
	\begin{align}
	\label{ineq:lower-bound-on-excess-risk-for-estimate-of-epoch-m}
	\begin{split}
		&\E_{x\sim D_{\Xscript}}\E_{a\sim\Unif(\A)}[(\hatf_{m+1}(x,a)-\hatf^*(x,a))^2]\\
		&\leq 2\E_{(x,r)\sim D}\E_{a\sim\Unif(\A)}[(\hatf_{m+1}(x,a)-r(a))^2  - (\hatf^*(x,a)-r(a))^2].
	\end{split}
	\end{align}
	Therefore, putting everything together (see \cref{ineq:uniform-population-bound-on-estimate-of-epoch-m} and \cref{ineq:lower-bound-on-excess-risk-for-estimate-of-epoch-m}), with probability at least $1-\delta'$ we have:
	\begin{align*}
		\E_{x\sim D_{\Xscript}}\E_{a\sim\Unif(\A)}[(\hatf_{m+1}(x,a)-\hatf^*(x,a))^2] \leq \frac{2L_1\ln^{\rho'}(|S_m'|)\ln(1/\delta')\comp(\F)}{|S_m'|^{\rho}}.
	\end{align*}
	Hence the first inequality in \Cref{lem:ConReg}  follows from noting that $|S_m'|=\lceil\epsilon(\tau_m-\tau_{m-1})\rceil$, and choosing $C_4=2L_1(2+\ln(12))$.

	Note that $\F$ is convex, $\hatf^*$ has no population excess risk with respect to the distribution generated from picking actions uniformly at random among estimators in $\F$, and note that $S_m'$ is generated by sampling actions uniformly at random. Hence from \Cref{ass:main-assumption}, with probability at least $1-\delta'$, we get that:
	\begin{align*}
	 \frac{1}{|S_m'|}\sum_{(x,a,r(a))\in S_m'} (\hatf^*(x,a)-r(a))^2 - \frac{1}{|S_m'|}\min_{g\in\F} \sum_{(x,a,r(a))\in S_m'} (g(x,a)-r(a))^2 \\
   \leq  \frac{(3C/2)\ln^{\rho'}(|S_m'|)\ln(1/\delta')\comp(\F)}{|S_m'|^{\rho}}.
	\end{align*}
	Therefore by choosing $C_1 \geq 3C/2$, with probability at least $1-\delta'$, we get that $\hatf^*\in\F'$. Now recall that:
	\begin{align*}
		\hatf_{m+1} \in \arg\min_{f\in\F'} \frac{1}{|S_m|} \sum_{(x,a,r(a))\in S_m} (f(x,a)-r(a))^2
	\end{align*}
	That is, $\hatf_{m+1}$ has no empirical excess risk with respect to the empirical data $S_m'$ among estimators in $\F'$. Also note that $\F'$ is convex subset of $\F$, and $S_m$ is generated by sampling actions according to the action selection kernel $p_m$. Hence from \Cref{ass:main-assumption}, with probability at least $1-\delta'$, we get that:
	\begin{align}
	\label{ineq:pm-population-bound-on-estimate-of-epoch-m}
	\begin{split}
		\E_{(x,r)\sim D}\E_{a\sim p_m(\cdot|x)}[(\hatf_{m+1}(x,a)-r(a))^2] - \min_{f\in\F'} \E_{(x,r)\sim D}\E_{a\sim p_m(\cdot|x)}[(f(x,a)-r(a))^2]\\
    \leq \frac{2C\ln^{\rho'}(|S_m|)\ln(1/\delta')\comp(\F)}{|S_m|^{\rho}}.
    \end{split}
	\end{align}
	Hence by taking union bound so that \cref{ineq:pm-population-bound-on-estimate-of-epoch-m} holds and $\hatf^*\in\F'$, with probability at least $1-2\delta'$, we have:
	\begin{align*}
	\E_{(x,r)\sim D}\E_{a\sim p_m(\cdot|x)}[(\hatf_{m+1}(x,a)-r(a))^2] -  \E_{(x,r)\sim D}\E_{a\sim p_m(\cdot|x)}[(\hatf^*(x,a)-r(a))^2]\\
  \leq \frac{2C\ln^{\rho'}(|S_m|)\ln(1/\delta')\comp(\F)}{|S_m|^{\rho}}.
	\end{align*}
	Recall that $B$ is the worst case excess risk for $\hatf^*$ under any kernel. Therefore, with probability at least $1-2\delta'$, we have:
	\begin{align*}
		&\E_{x\sim D_{\Xscript}}\E_{a\sim p_m(\cdot|x)}[(\hatf_{m+1}(x,a)-f^*(x,a))^2]\\
		& = \E_{(x,r)\sim D}\E_{a\sim p_m(\cdot|x)}[(\hatf_{m+1}(x,a)-r(a))^2 - (f^*(x,a)-r(a))^2]\\
		& \leq \E_{(x,r)\sim D}\E_{a\sim p_m(\cdot|x)}[(\hatf^*(x,a)-r(a))^2- (f^*(x,a)-r(a))^2] + \frac{2C\ln^{\rho'}(|S_m|)\ln(1/\delta')\comp(\F)}{|S_m|^{\rho}} \\
		& = \E_{x\sim D_{\Xscript}}\E_{a\sim p_m(\cdot|x)}[(\hatf^*(x,a)-f^*(x,a))^2] + \frac{2C\ln^{\rho'}(|S_m|)\ln(1/\delta')\comp(\F)}{|S_m|^{\rho}}\\
		& \leq B + \frac{2C\ln^{\rho'}(|S_m|)\ln(1/\delta')\comp(\F)}{|S_m|^{\rho}}.
	\end{align*}
  For any epoch $m\geq 1$, note that $\tau_m-\tau_{m-1}\geq \tau_1 \geq 4$. Therefore since $\epsilon< 0.5$, we get that:
	$$|S_m|= \tau_m-\tau_{m-1}-\lceil \epsilon(\tau_m-\tau_{m-1}) \rceil \geq \frac{1}{4}(\tau_m-\tau_{m-1}). $$
	Hence the second inequality in \Cref{lem:ConReg}  follows from choosing an appropriate value for $C_5=2C\times4^{\rho}\times (2+\ln(12))$. Taking union bound, we finally note that both inequalities in \cref{lem:ConReg} hold for all epochs with probability at least:
	\begin{align*}
		1-\sum_{m=1}^{\infty} 3\frac{\delta}{12m^2} \geq 1 - \frac{\delta (\pi^2/6)}{4} \geq 1-\delta/2.
	\end{align*}
\end{proof}

\paragraph{Additional notation}

For compactness of notation, define the following event:
\begin{equation}
    \label{eq:w-event}
    \begin{aligned}
      \event := \Bigg\{ &\forall m \geq 1, \; \E_{x\sim D_{\Xscript}}\E_{a\sim \Unif(\A)}[(\hatf_{m+1}(x,a)-\hatf^*(x,a))^2] \leq \frac{C_4\ln^{\rho'}(\tau_m-\tau_{m-1})\ln(m/\delta)\comp(\F)}{(\epsilon(\tau_m-\tau_{m-1}))^{\rho}}, \\
      & \E_{x\sim D_{\Xscript}}\E_{a\sim p_m(\cdot|x)}[(\hatf_{m+1}(x,a)-f^*(x,a))^2] \leq B +  \frac{C_5\ln^{\rho'}(\tau_m-\tau_{m-1})\ln(m/\delta)\comp(\F)}{(\tau_m-\tau_{m-1})^{\rho}} \Bigg\},
    \end{aligned}
\end{equation}
for two constants $C_4$ and $C_5$ that were defined in \Cref{lem:ConReg}.

\subsection{Bounding prediction error of implicit rewards}
\label{app:bounding-prediction-error}

For any policy, \Cref{lem:reg-est-accuracy} bounds the prediction error of implicit reward estimate of the policy at every epoch. This Lemma and its proof are similar to Lemma 7 in \cite{simchi2020bypassing}.

\begin{restatable}[Accuracy of implicit policy estimate]{lemma}{lemImpPolEval}
\label{lem:reg-est-accuracy}
Suppose $C_3\leq 1/(4C_5)$ and suppose the event $\event$ from \eqref{eq:w-event} holds. Then, for all policies $\pi$ and epoch $m\geq 1$, we have:
\begin{align*}
    |R_{\hatf_{m+1}}(\pi)-R(\pi)| \leq \sqrt{V(p_m,\pi)}\sqrt{B} + \frac{\sqrt{V(p_m,\pi)}\sqrt{K}}{2\gamma_{m+1}}
\end{align*}
\end{restatable}
\begin{proof}
For any policy $\pi$ and epoch $m\geq 1$, note that:
\begin{align*}
    &|R_{\hatf_{m+1}}(\pi)-R(\pi)|^2\\
    \leq & \Big(\E_{x\sim D_{\Xscript}}\Big[\Big|\hatf_{m+1}(x,\pi(x))-f^*(x,\pi(x))\Big|\Big] \Big)^2\\
    = & \Bigg(\E_{x\sim D_{\Xscript}}\Bigg[\sqrt{\frac{1}{p_m(\pi(x)|x)}p_m(\pi(x)|x)\Big(\hatf_{m+1}(x,\pi(x))-f^*(x,\pi(x))\Big)^2}\Bigg] \Bigg)^2\\
    \leq & \Bigg(\E_{x\sim D_{\Xscript}}\Bigg[\sqrt{\frac{1}{p_m(\pi(x)|x)}\E_{a\sim p_m(\cdot|x)}\Big[\hatf_{m+1}(x,a)-f^*(x,a)\Big]^2}\Bigg] \Bigg)^2\\
    \leq & \E_{x\sim D_{\Xscript}}\Bigg[\frac{1}{p_m(\pi(x)|x)} \Bigg] \E_{x\sim D_{\Xscript}} \E_{a\sim p_m(\cdot|x)}\Big[\Big(\hatf_{m+1}(x,a)-f^*(x,a)\Big)^2\Big]\\
    \leq & V(p_m, \pi)\bigg(B + \frac{C_5\ln^{\rho'}(\tau_m-\tau_{m-1})\ln(1/\delta')\comp(\F)}{(\tau_m-\tau_{m-1})^{\rho}} \bigg).
\end{align*}
The first inequality follows from Jensen's inequality, the second inequality is straight forward, the third inequality follows from Cauchy-Schwarz inequality, and the last inequality follows from assuming that $\event$ from \eqref{eq:w-event} holds. Now from the sub-additive property of square-root, we get:
\begin{align*}
    |R_{\hatf_{m+1}}(\pi)-R(\pi)| & \leq  \sqrt{V(p_m,\pi)}\bigg(\sqrt{B}+\sqrt{\frac{C_5\ln^{\rho'}(\tau_m-\tau_{m-1})\ln(m/\delta)\comp(\F)}{(\tau_m-\tau_{m-1})^{\rho}}}\bigg)\\
    & \leq \sqrt{V(p_m,\pi)}\sqrt{B} + \frac{\sqrt{V(p_m,\pi)}\sqrt{K}}{2\gamma_{m+1}}.
\end{align*}
Where the last inequality follows from the choice of $\gamma_{m+1}$ and from assuming that $C_3 \leq 1/(4C_5)$.
\end{proof}

\subsection{Bounding decisional divergence}
\label{app:bounding-v}

At any epoch $m$, \Cref{lem:boundVopt} bounds the decisional divergence between the active policy at that epoch ($Q_m$) and the policy induced by the best estimator ($\pi_{\hatf^*}$). This implies that even as the active policy is less explorative, $Q_m$ is not very far from $\pi_{\hatf^*}$ and hence eventually converges to it.

\begin{restatable}[Action selection kernels are always close to target policy]{lemma}{lemboundVopt}
\label{lem:boundVopt}
Suppose the event $\event$ from \eqref{eq:w-event} holds. Then there exists a positive constant $C_6$ such that, for any epoch $m\geq 1$, we have:
$$ V(p_m,\pi_{\hatf^*}) \leq \frac{C_6K}{\sqrt{\epsilon^{\rho}}} $$
\end{restatable}
\begin{proof}
Since the action selection kernel $p_1(\cdot|x)$ draws actions uniformly at random for all $x\in\Xscript$, we have that $V(p_1,\pi_{\hatf^*})=K$. Hence, by choosing $C_6\geq1$, we get that $V(p_1,\pi_{\hatf^*})\leq C_6K/\sqrt{\epsilon^{\rho}}$. Now consider any epoch $m\geq 2$. Note that from the definition of $\pi_{\hatf^*}$, for any context $x\in\Xscript$ we get: $$\hatf^*(x,\pi_{\hatf^*}(x)) = \max_{a\in\A}\hatf^*(x,a) \geq \hatf^*(x,\pi_{\hatf_m}(x)).$$
Hence from the above inequality, for any context $x\in\Xscript$ we get:
\begin{align*}
    &\hatf_{m}(x,\pi_{\hatf_{m}}(x))- \hatf_{m}(x,\pi_{\hatf^*}(x))\\
    & = \big(\hatf_{m}(x,\pi_{\hatf_{m}}(x))- \hatf^*(x,\pi_{\hatf_{m}}(x))\big) + \big(\hatf^*(x,\pi_{\hatf_{m}}(x))- \hatf_{m}(x,\pi_{\hatf^*}(x))\big)\\
    &\leq \big(\hatf_{m}(x,\pi_{\hatf_{m}}(x))- \hatf^*(x,\pi_{\hatf_{m}}(x))\big) + \big(\hatf^*(x,\pi_{\hatf^*}(x))- \hatf_{m}(x,\pi_{\hatf^*}(x))\big)\\
    &\leq \big|\hatf_{m}(x,\pi_{\hatf_{m}}(x))- \hatf^*(x,\pi_{\hatf_{m}}(x))\big| + \big|\hatf^*(x,\pi_{\hatf^*}(x))- \hatf_{m}(x,\pi_{\hatf^*}(x))\big|\\
    & \leq 2\max_{a\in\A} \big|\hatf^*(x,a)- \hatf_{m}(x,a)\big|.
\end{align*}
Now from \Cref{lem:boundV}, the above inequality, and Jensen's inequality, we get:
\begin{align*}
    &V(p_m,\pi_{\hatf^*}) \leq \E_{x\sim\Xscript}\Big[K + \gamma_m\big(\hatf_{m}(x,\pi_{\hatf_{m}}(x))- \hatf_{m}(x,\pi_{\hatf^*}(x))\big) \Big]\\
    & \leq K + 2\gamma_m\E_{x\sim\Xscript}\Big[ \max_{a\in\A} \big|\hatf^*(x,a)- \hatf_{m}(x,a)\big| \Big]  \\
    & \leq K + 2\gamma_m \sqrt{\E_{x\sim\Xscript}\Big[ \max_{a\in\A} \big|\hatf^*(x,a)- \hatf_{m}(x,a)\big|^2 \Big]}  \\
    &\leq K + 2\gamma_m\sqrt{ \sum_{a\in\A} \E_{x\sim\Xscript}\Big[\big(\hatf_{m}(x,a)- \hatf^*(x,a)\big)^2\Big]}.
\end{align*}
Now let $C_6 = 1 + 2\sqrt{C_3C_4}$. From the above inequality, we further get:
\begin{align*}
    &V(p_m,\pi_{\hatf^*}) \leq K + 2\gamma_m\sqrt{ \sum_{a\in\A} \E_{x\sim\Xscript}\Big[\big(\hatf_{m}(x,a)- \hatf^*(x,a)\big)^2\Big]}\\
    &\leq K + 2\gamma_m\sqrt{K\frac{C_4\ln^{\rho'}(\tau_{m-1}-\tau_{m-2})\ln((m-1)/\delta)\comp(\F)}{(\epsilon(\tau_{m-1}-\tau_{m-2}))^{\rho}}} \\
    & = K + 2K\sqrt{\frac{C_3C_4}{\epsilon^{\rho}}} \leq \frac{C_6K}{\sqrt{\epsilon^{\rho}}}.
\end{align*}
Where the second inequality follows from the assumption that $\event$ holds. And the last inequality follows from our choice of $C_6$.
\end{proof}

Lemma \ref{lem:boundVall} shows that for any policy $\pi$ and epoch $m$, if the decisional divergence between $Q_m$ and $\pi$ was large, then the decisional divergence between $Q_{m+1}$ and $\pi$ must also be large. Hence the Lemma shows that the active phase of $\GeneralizedFalcon$ stops exploring in a stable manner.

\begin{restatable}[Do not pick up policies that you drop]{lemma}{lemboundVall}
\label{lem:boundVall}
Suppose the event $\event$ defined in \eqref{eq:w-event} holds, and $\delta \leq 0.5$. Then there exists a positive constant $C_7$ such that, for all policies $\pi$ and epochs $m$, we have:
$$ V(p_m,\pi) \leq \frac{C_7K}{\sqrt{\epsilon^{\rho}}} + V(p_{m+1},\pi). $$
\end{restatable}
\begin{proof}
Consider any policy $\pi$. Since the action selection kernel $p_1(\cdot|x)$ draws actions uniformly at random for all $x$, we have that $V(p_1,\pi)=K$. Hence, by choosing $C_7\geq1$, we get that $V(p_1,,\pi) \leq \frac{C_7K}{\sqrt{\epsilon^{\rho}}} + V(p_2,\pi)$. Now consider any epoch $m\geq 2$. For any context $x\in\Xscript$, we get:
\begin{align*}
  \hatf_{m+1}(x,\pi_{\hatf_{m+1}}(x)) = \max_{a\in\A}\hatf_{m+1}(x,a) \geq
  \begin{cases}
  \hatf_{m+1}(x,\pi_{\hatf_{m}}(x))\\
  \hatf_{m+1}(x,\pi(x)).
  \end{cases}
\end{align*}
From \Cref{lem:boundV}, the fact that $\gamma_{m+1}\geq \gamma_m$, and the above inequality, we get:
\begin{align}
\label{ineq:V-diff-bound}
\begin{split}
    & V(p_m,\pi) - K - V(p_{m+1},\pi)\\
    & \leq \gamma_{m}\E_{x\sim D_{\Xscript}}[\hatf_{m}(x,\pi_{\hatf_{m}}(x)) - \hatf_{m}(x,\pi(x))] - \gamma_{m+1}\E_{x\sim D_{\Xscript}}[\hatf_{m+1}(x,\pi_{\hatf_{m+1}}(x)) -  \hatf_{m+1}(x,\pi(x))]\\
    & \leq \gamma_{m}\E_{x\sim D_{\Xscript}}[(\hatf_{m}(x,\pi_{\hatf_{m}}(x)) - \hatf_{m+1}(x,\pi_{\hatf_{m+1}}(x))) + (\hatf_{m+1}(x,\pi(x)) - \hatf_{m}(x,\pi(x)))]\\
    & \leq \gamma_{m}\E_{x\sim D_{\Xscript}}[(\hatf_{m}(x,\pi_{\hatf_{m}}(x)) - \hatf_{m+1}(x,\pi_{\hatf_{m}}(x))) + (\hatf_{m+1}(x,\pi(x)) - \hatf_{m}(x,\pi(x)))]\\
    & \leq 2\gamma_m\E_{x\sim D_{\Xscript}} \Big[\max_{a} |\hatf_{m+1}(x,a) - \hatf_{m}(x,a)| \Big].
\end{split}
\end{align}

Also note that from Jensen's inequality, we get:
\begin{align}
\label{ineq:epoch-est-diff-bound}
\begin{split}
    & \E_{x\sim D_{\Xscript}} \Big[\max_{a} |\hatf_{m+1}(x,a) - \hatf_{m}(x,a)| \Big]\\
    & \leq \sqrt{\E_{x\sim D_{\Xscript}} \Big[\max_{a} |\hatf_{m+1}(x,a) - \hatf_{m}(x,a)|^2 \Big]}\\
    & \leq \sqrt{\sum_a\E_{x\sim D_{\Xscript}} \Big[(\hatf_{m+1}(x,a) - \hatf_{m}(x,a))^2 \Big]}\\
    & = \sqrt{K \; \E_{x\sim D_{\Xscript}}\E_{a\sim\Unif(\A)} \Big[(\hatf_{m+1}(x,a) - \hatf_{m}(x,a))^2 \Big]}\\
    & \leq \sqrt{2K \; \E_{x\sim D_{\Xscript}}\E_{a\sim\Unif(\A)} \Big[(\hatf_{m+1}(x,a) - \hatf^*(x,a))^2 + (\hatf^*(x,a) - \hatf_{m}(x,a))^2  \Big]}\\
    & \leq \sqrt{\frac{4KC_4\ln^{\rho'}(\tau_m-\tau_{m-1})\ln(m/\delta)\comp(\F)}{(\epsilon(\tau_{m-1}-\tau_{m-2}))^{\rho}}}.
\end{split}
\end{align}
The second last inequality follows from the identity that for any two real numbers $u,v$, $(u+v)^2\leq 2(u^2+v^2)$. The last inequality follows from the assumption that $\event$ holds, the fact that $m\geq m-1$, and the fact that epoch lengths are non-decreasing (i.e. $\tau_m - \tau_{m-1} \geq \tau_{m-1} - \tau_{m-1}$).  Now, by combining \Cref{ineq:V-diff-bound} and \Cref{ineq:epoch-est-diff-bound}, we get:
\begin{align*}
    & V(p_m,\pi) \leq V(p_{m+1},\pi) + K + 4\gamma_m\sqrt{\frac{KC_4\ln^{\rho'}(\tau_m-\tau_{m-1})\ln(m/\delta)\comp(\F)}{(\epsilon(\tau_{m-1}-\tau_{m-2}))^{\rho}}}\\
    & = V(p_{m+1},\pi) + K + 4K\sqrt{\frac{C_3C_4}{\epsilon^{\rho}} \frac{\ln^{\rho'}(\tau_m-\tau_{m-1})}{\ln^{\rho'}(\tau_{m-1}-\tau_{m-2})} \frac{\ln(m/\delta)}{\ln((m-1)/\delta)}  } \\
    & \leq  V(p_{m+1},\pi) + \frac{C_7K}{\sqrt{\epsilon^{\rho}}}.
\end{align*}
Where the last inequality follows from choosing $C_7 = 1 + 4\sqrt{2^{1+\rho'}C_3C_4}$, and from the fact that for $m\geq 2$ and $\delta\leq 0.5$ we have: $\frac{\ln(m/\delta)}{\ln((m-1)/\delta)} \leq 2$, and $\frac{\ln^{\rho'}(\tau_m-\tau_{m-1})}{\ln^{\rho'}(\tau_{m-1}-\tau_{m-2})}\leq \frac{\ln^{\rho'}(\tau_{m-1})}{\ln^{\rho'}(\tau_{m-1}/2)}\leq 2^{\rho'}$.
\end{proof}

\subsection{Bounding prediction error of implicit regret}
\label{app:bounding-prediction-error-implicit-regret}

For any policy, \Cref{lem:policy-reg-bound} bounds the prediction error of implicit regret estimate of the policy at every epoch. This Lemma and its proof are similar to Lemma 8 in \cite{simchi2020bypassing}.

\begin{restatable}[Bounds on implicit estimates of policy regret]{lemma}{lemboundReg}
\label{lem:policy-reg-bound}
Suppose the event $\event$ defined in \eqref{eq:w-event} holds, and $\delta \leq 0.5$. Then there exists positive constants $C_0, C_8, C_9$ such that, for all policies $\pi$ and epochs $m$, we have:
\begin{align*}
    \Reg(\pi) \leq 2\Reg_{\hatf_{m}}(\pi) + \frac{C_0K}{\gamma_m} + C_8\sqrt{\frac{KB}{\sqrt{\epsilon^{\rho}}}} + C_9\sqrt{V(p_m,\pi)B}\\
    \Reg_{\hatf_{m}}(\pi) \leq 2\Reg(\pi) + \frac{C_0K}{\gamma_m} + C_8\sqrt{\frac{KB}{\sqrt{\epsilon^{\rho}}}} + C_9\sqrt{V(p_m,\pi)B}
\end{align*}
\end{restatable}
\begin{proof}
We will prove this by induction. Let $C_0$ be a positive constant such that $C_0 \geq 1 \geq \gamma_1/K$. The base case then follows from the fact that for all policies $\pi$, we have:
\begin{align*}
    \Reg(\pi) \leq 1 \leq C_0K/\gamma_1\\
    \Reg_{\hatf_{1}}(\pi) \leq 1 \leq C_0K/\gamma_1.
\end{align*}
For the inductive step, fix some $m\geq 1$. Assume for all policies $\pi$, we have:
\begin{align}
\label{eq:IH}
    \Reg(\pi) \leq 2\Reg_{\hatf_{m}}(\pi) + \frac{C_0K}{\gamma_m} + C_8\sqrt{\frac{KB}{\sqrt{\epsilon^{\rho}}}} + C_9\sqrt{V(p_m,\pi)B} \nonumber \\
    \Reg_{\hatf_{m}}(\pi) \leq 2\Reg(\pi) + \frac{C_0K}{\gamma_m} + C_8\sqrt{\frac{KB}{\sqrt{\epsilon^{\rho}}}} + C_9\sqrt{V(p_m,\pi)B}
\end{align}
Note that:
\begin{align*}
    &\Reg(\pi) - \Reg_{\hatf_{m+1}}(\pi) \\
    =& \Big(R(\pi_{f^*}) - R(\pi)\Big) - \Big(R_{\hatf_{m+1}}(\pi_{\hatf_{m+1}}) - R_{\hatf_{m+1}}(\pi)\Big) \\
    \leq & \Big(R(\pi_{\hatf^*}) - R(\pi)\Big) - \Big(R_{\hatf_{m+1}}(\pi_{\hatf_{m+1}}) - R_{\hatf_{m+1}}(\pi)\Big) + 2\sqrt{B} \\
    \leq & \Big(R(\pi_{\hatf^*}) - R(\pi)\Big) - \Big(R_{\hatf_{m+1}}(\pi_{\hatf^*}) - R_{\hatf_{m+1}}(\pi)\Big) + 2\sqrt{B} \\
    \leq & |R(\pi_{\hatf^*}) - R_{\hatf_{m+1}}(\pi_{\hatf^*})| + |R(\pi) - R_{\hatf_{m+1}}(\pi))| + 2\sqrt{B} .
\end{align*}
Where the first inequality follows from \Cref{lem:reg-of-best-pred}, and the second inequality follows from the definition of $\pi_{\hatf_{m+1}}$ which gives us that $  R_{\hatf_{m+1}}(\pi_{\hatf^*}) \leq R_{\hatf_{m+1}}(\pi_{\hatf_{m+1}})$. Now, further simplifying the above inequality we get:
\begin{align}
\label{eq:inductive-bound-for-pi}
\begin{split}
    &\Reg(\pi) - \Reg_{\hatf_{m+1}}(\pi) \\
    \leq & |R(\pi_{\hatf^*}) - R_{\hatf_{m+1}}(\pi_{\hatf^*})| + |R(\pi) - R_{\hatf_{m+1}}(\pi))| + 2\sqrt{B} \\
    \leq &  \sqrt{V(p_m,\pi_{\hatf^*})}\sqrt{B} + \frac{\sqrt{V(p_m,\pi_{\hatf^*})}\sqrt{K}}{2\gamma_{m+1}} + \sqrt{V(p_m,\pi)}\sqrt{B} + \frac{\sqrt{V(p_m,\pi)}\sqrt{K}}{2\gamma_{m+1}} + 2\sqrt{B} \\
    \leq & \frac{5K}{8\gamma_{m+1}} + \frac{V(p_m,\pi_{\hatf^*})}{5\gamma_{m+1}} + \frac{V(p_m,\pi)}{5\gamma_{m+1}} + \sqrt{B}\bigg(\sqrt{V(p_m,\pi_{\hatf^*})} + \sqrt{V(p_m,\pi)} + 2 \bigg) \\
    \leq & \frac{5K}{8\gamma_{m+1}} + \frac{V(p_m,\pi_{\hatf^*})}{5\gamma_{m+1}} + \frac{V(p_m,\pi)}{5\gamma_{m+1}} + \Big(\sqrt{C_6}+\sqrt{C_7}\Big)\sqrt{\frac{BK}{\sqrt{\epsilon^{\rho}}}} + \sqrt{B}\sqrt{V(p_{m+1},\pi)} + 2\sqrt{B}.
\end{split}
\end{align}
Where the second inequality follow from \Cref{lem:reg-est-accuracy}, the third inequality is an application of Cauchy-Schwarz inequality, and the last inequality follows from Lemmas \ref{lem:boundVopt} and \ref{lem:boundVall}. Now note that:
\begin{align}
\label{eq:inductive-hyp-for-opt-pi}
    &\frac{V(p_m,\pi_{\hatf^*})}{5\gamma_{m+1}} \leq \frac{K+\gamma_m\Reg_{\hatf_m}(\pi_{\hatf^*})}{5\gamma_{m+1}} \nonumber\\
    &\leq \frac{K+\gamma_m\Big( 2\Reg(\pi_{\hatf^*}) + \frac{C_0K}{\gamma_m} + C_8\sqrt{\frac{KB}{\sqrt{\epsilon^{\rho}}}} + C_9\sqrt{V(p_m,\pi_{\hatf^*})B}  \Big)}{5\gamma_{m+1}} \nonumber \\
    & \leq \frac{K(1+C_0)}{5\gamma_{m+1}} + \frac{2\Reg(\pi_{\hatf^*})}{5} + \frac{C_8}{5} \sqrt{\frac{KB}{\sqrt{\epsilon^{\rho}}}} + \frac{C_9}{5}\sqrt{V(p_m,\pi_{\hatf^*})B} \nonumber \\
    & \leq \frac{K(1+C_0)}{5\gamma_{m+1}} + \frac{4\sqrt{B}}{5} + \frac{C_8 + C_9\sqrt{C_6}}{5} \sqrt{\frac{KB}{\sqrt{\epsilon^{\rho}}}}.
\end{align}
Where the first inequality follows from \Cref{lem:boundV}, the second inequality follows from \Cref{eq:IH}, and the last inequality follows from Lemmas \ref{lem:reg-of-best-pred} and \ref{lem:boundVopt}. Similarly note that:
\begin{align}
\label{eq:inductive-hyp-for-pi}
    &\frac{V(p_m,\pi)}{5\gamma_{m+1}} \leq \frac{K+\gamma_m\Reg_{\hatf_m}(\pi)}{5\gamma_{m+1}} \nonumber\\
    &\leq \frac{K+\gamma_m\Big( 2\Reg(\pi) + \frac{C_0K}{\gamma_m} + C_8\sqrt{\frac{KB}{\sqrt{\epsilon^{\rho}}}} + C_9\sqrt{V(p_m,\pi)B}  \Big)}{5\gamma_{m+1}} \nonumber \\
    & \leq \frac{K(1+C_0)}{5\gamma_{m+1}} + \frac{2\Reg(\pi)}{5} + \frac{C_8}{5} \sqrt{\frac{KB}{\sqrt{\epsilon^{\rho}}}} + \frac{C_9}{5}\sqrt{V(p_m,\pi)B} \nonumber \\
    & \leq \frac{K(1+C_0)}{5\gamma_{m+1}} + \frac{2\Reg(\pi)}{5} + \frac{C_8 + C_9\sqrt{C_7}}{5} \sqrt{\frac{KB}{\sqrt{\epsilon^{\rho}}}} + \frac{C_9}{5}\sqrt{V(p_{m+1},\pi)B}.
\end{align}
Where the first inequality follows from \Cref{lem:boundV}, the second inequality follows from \Cref{eq:IH}, and the last inequality follows from \Cref{lem:boundVall}.
Now from combining \Cref{eq:inductive-bound-for-pi}, \Cref{eq:inductive-hyp-for-opt-pi}, and \Cref{eq:inductive-hyp-for-pi}, we get:
\begin{align*}
    \Reg(\pi) - \Reg_{\hatf_{m+1}}(\pi) \leq &\frac{5K}{8\gamma_{m+1}} + \frac{2K(1+C_0)}{5\gamma_{m+1}} + \frac{2\Reg(\pi)}{5} + \frac{14}{5}\sqrt{B}\\
    &+ \frac{2C_8 + (C_9+5)(\sqrt{C_6}+\sqrt{C_7})}{5}\sqrt{\frac{KB}{\sqrt{\epsilon^{\rho}}}} + \frac{C_9+5}{5}\sqrt{V(p_{m+1},\pi)B}
\end{align*}
Which implies:
\begin{align*}
    \Reg(\pi)\leq &  \frac{5}{3}\Reg_{\hatf_{m+1}}(\pi) + \frac{K(2C_0+5.125)}{3\gamma_{m+1}} + \frac{14}{3}\sqrt{B} \\
    &+ \frac{2C_8 + (C_9+5)(\sqrt{C_6}+\sqrt{C_7})}{3}\sqrt{\frac{KB}{\sqrt{\epsilon^{\rho}}}} + \frac{C_9+5}{3}\sqrt{V(p_{m+1},\pi)B}
\end{align*}
Now choosing constants so that $C_0\geq 5.125$, $C_9\geq 2.5$, and $C_8\geq (C_9+5)(\sqrt{C_6}+\sqrt{C_7})$. The above inequality then gives us:
\begin{align}
\label{eq:IH-implication}
    \Reg(\pi)\leq &  2\Reg_{\hatf_{m+1}}(\pi) + \frac{C_0K}{\gamma_{m+1}} + C_8\sqrt{\frac{KB}{\sqrt{\epsilon^{\rho}}}} + C_9\sqrt{V(p_{m+1},\pi)B}.
\end{align}
Hence from our induction hypothesis (\Cref{eq:IH}), we get \Cref{eq:IH-implication}, which provides the required upper bound on $\Reg(\pi)$ in terms of $\Reg_{\hatf_{m+1}}(\pi)$. To complete the inductive argument, we need to show the corresponding upper bound on $\Reg_{\hatf_{m+1}}(\pi)$. Similar to \Cref{eq:inductive-bound-for-pi}, we get:
\begin{align}
\label{eq:inductive-bound-for-pi-est}
    &\Reg_{\hatf_{m+1}}(\pi) - \Reg(\pi) \nonumber\\
    =& \Big(R_{\hatf_{m+1}}(\pi_{\hatf_{m+1}}) - R_{\hatf_{m+1}}(\pi)\Big) - \Big(R(\pi_{f^*}) - R(\pi)\Big)  \nonumber\\
    \leq & \Big(R_{\hatf_{m+1}}(\pi_{\hatf_{m+1}}) - R_{\hatf_{m+1}}(\pi)\Big) - \Big(R(\pi_{\hatf_{m+1}}) - R(\pi)\Big)  \nonumber\\
    \leq & |R(\pi_{\hatf_{m+1}}) - R_{\hatf_{m+1}}(\pi_{\hatf_{m+1}})| + |R(\pi) - R_{\hatf_{m+1}}(\pi))| \nonumber\\
    \leq &  \sqrt{V(p_m,\pi_{\hatf_{m+1}})}\sqrt{B} + \frac{\sqrt{V(p_m,\pi_{\hatf_{m+1}})}\sqrt{K}}{2\gamma_{m+1}} + \sqrt{V(p_m,\pi)}\sqrt{B} + \frac{\sqrt{V(p_m,\pi)}\sqrt{K}}{2\gamma_{m+1}} \nonumber \\
    \leq & \frac{5K}{8\gamma_{m+1}} + \frac{V(p_m,\pi_{\hatf_{m+1}})}{5\gamma_{m+1}} + \frac{V(p_m,\pi)}{5\gamma_{m+1}} + \sqrt{B}\bigg(\sqrt{V(p_m,\pi_{\hatf_{m+1}})} + \sqrt{V(p_m,\pi)}\bigg) \nonumber \\
    \leq & \frac{5K}{8\gamma_{m+1}} + \frac{V(p_m,\pi_{\hatf_{m+1}})}{5\gamma_{m+1}} + \frac{V(p_m,\pi)}{5\gamma_{m+1}} + 2\sqrt{\frac{C_7BK}{\sqrt{\epsilon^{\rho}}}} + \sqrt{B}\bigg(\sqrt{V(p_{m+1},\pi_{\hatf_{m+1}})} + \sqrt{V(p_{m+1},\pi)}\bigg).
\end{align}
Where the first inequality follows from the definition of $\pi_{f^*}$, the second inequality is straight forward, the third inequality follows from \Cref{lem:reg-est-accuracy}, the forth inequality is an application of Cauchy-Schwarz inequality, and the last inequality follows from \Cref{lem:boundVall}. Similar to \Cref{eq:inductive-hyp-for-opt-pi}, we get:
\begin{align}
\label{eq:inductive-hyp-for-mplusone-est}
    &\frac{V(p_m,\pi_{\hatf_{m+1}})}{5\gamma_{m+1}} \leq \frac{K+\gamma_m\Reg_{\hatf_m}(\pi_{\hatf_{m+1}})}{5\gamma_{m+1}} \nonumber\\
    &\leq \frac{K+\gamma_m\Big( 2\Reg(\pi_{\hatf_{m+1}}) + \frac{C_0K}{\gamma_m} + C_8\sqrt{\frac{KB}{\sqrt{\epsilon^{\rho}}}} + C_9\sqrt{V(p_m,\pi_{\hatf_{m+1}})B}  \Big)}{5\gamma_{m+1}} \nonumber \\
    & \leq \frac{K(1+C_0)}{5\gamma_{m+1}} + \frac{2\Reg(\pi_{\hatf_{m+1}})}{5} + \frac{C_8}{5} \sqrt{\frac{KB}{\sqrt{\epsilon^{\rho}}}} + \frac{C_9}{5}\sqrt{V(p_m,\pi_{\hatf_{m+1}})B} \nonumber \\
    & \leq \frac{K(1+C_0)}{5\gamma_{m+1}} + \frac{2\Reg(\pi_{\hatf_{m+1}})}{5} + \frac{C_8 + C_9\sqrt{C_7}}{5} \sqrt{\frac{KB}{\sqrt{\epsilon^{\rho}}}} + \frac{C_9}{5}\sqrt{V(p_{m+1},\pi_{\hatf_{m+1}})B} \nonumber\\
    & \leq \frac{K(1+3C_0)}{5\gamma_{m+1}} + \frac{3C_8 + C_9\sqrt{C_7}}{5} \sqrt{\frac{KB}{\sqrt{\epsilon^{\rho}}}} + \frac{3C_9}{5}\sqrt{V(p_{m+1},\pi_{\hatf_{m+1}})B}.
\end{align}
Where the first inequality follows from \Cref{lem:boundV}, the second inequality follows from \Cref{eq:IH}, the forth inequality follows from \Cref{lem:boundVall}, and the last inequality follows from \Cref{eq:IH-implication}.
Also note that:
\begin{align}
\label{eq:bound-on-expected-ipw-on-generator}
    V(p_{m+1},\pi_{\hatf_{m+1}}) \leq K+\gamma_{m+1}\Reg_{\hatf_{m+1}}(\pi_{\hatf_{m+1}}) = K.
\end{align}
Combining \Cref{eq:inductive-hyp-for-pi}, \Cref{eq:inductive-bound-for-pi-est}, \Cref{eq:inductive-hyp-for-mplusone-est}, and \Cref{eq:bound-on-expected-ipw-on-generator}, we get:
\begin{align*}
    \Reg_{\hatf_{m+1}}(\pi) \leq & \frac{7\Reg(\pi)}{5} + \frac{5K}{8\gamma_{m+1}} + \frac{2K(1+2C_0)}{5\gamma_{m+1}} + \frac{4C_8 + 2\sqrt{C_7}(C_9+5)}{5}\sqrt{\frac{KB}{\sqrt{\epsilon^{\rho}}}}\\
    & + \frac{C_9+5}{5}\sqrt{V(p_{m+1},\pi)B} + \frac{3C_9+5}{5}\sqrt{KB}.
\end{align*}
Now choosing constants so that $C_0\geq 2$, $C_9\geq 2.5$, and $C_8\geq 2\sqrt{C_7}(C_9+5)+(3C_9+5)$. The above inequality then gives us:
\begin{align}
    \Reg_{\hatf_{m+1}}(\pi) \leq & 2\Reg(\pi) + \frac{C_0K}{\gamma_{m+1}} + C_8\sqrt{\frac{KB}{\sqrt{\epsilon^{\rho}}}} + C_9\sqrt{V(p_{m+1},\pi)B}.
\end{align}
This completes the inductive step.
\end{proof}

\subsection{Bounding true regret}
\label{app:bounding-true-regret}

For any epoch $m$, \Cref{lem:QmRegTrue} bounds regret of the randomized policy $Q_m$.

\begin{restatable}[Action selection kernel has low true regret]{lemma}{lemQmRegTrue}
\label{lem:QmRegTrue}
Suppose the event $\event$ defined in \eqref{eq:w-event} holds, and $\delta \leq 0.5$. And let $C_{10} := C_8+C_9$. Then for all epochs $m$, we have:
$$ \sum_{\pi \in \Psi} Q_m(\pi)\Reg(\pi) \leq \frac{(2+C_0)K}{\gamma_m} + C_{10}\sqrt{\frac{KB}{\sqrt{\epsilon^{\rho}}}}.$$
\end{restatable}
\begin{proof}
Note that:
\begin{align*}
    &\sum_{\pi \in \Psi} Q_m(\pi)\Reg(\pi)\\
    &\leq \sum_{\pi \in \Psi} Q_m(\pi)\bigg( 2\Reg_{\hatf_{m}}(\pi) + \frac{C_0K}{\gamma_m} + C_8\sqrt{\frac{KB}{\sqrt{\epsilon^{\rho}}}} + C_9\sqrt{V(p_m,\pi)B} \bigg)\\
    & \leq \frac{2K}{\gamma_m} + \frac{C_0K}{\gamma_m} + C_8\sqrt{\frac{KB}{\sqrt{\epsilon^{\rho}}}} + C_9\sqrt{KB} \leq \frac{(2+C_0)K}{\gamma_m} + C_{10}\sqrt{\frac{KB}{\sqrt{\epsilon^{\rho}}}}.
\end{align*}
Where the first inequality follows from \Cref{lem:policy-reg-bound}, and the second inequality follows from Lemmas \ref{lem:QmRegEst} and \Cref{lem:QmRootV}.
\end{proof}

\subsection{Proof of Theorem \ref{thm:main-theorem}}
\label{app:proof-of-main-theorem}

We can now bound the cumulative regret of \GeneralizedFalcon. Fix some (possibly unknown) horizon $T$. Let $\activesteps\subseteq [T]$ be the set of time-steps that are in the active phase of some epoch. Similarly let $\passivesteps\subseteq [T]$ be the set of time-steps that are in the active phase of some epoch. Let $m(t)$ denote the epoch in which the time-step $t$ occurs. For each round $t\in\{1,2,\dots,T\}$, define: $$M_t:=r_t(\pi^*(x_t))-r_t(a_t)-\sum_{\pi\in\Psi} Q_{m(t)}(\pi)\Reg(\pi).$$
Recall that from \Cref{lem:conditional-reward}, for all $t\in\activesteps$ we have:
$$ \E_{x_t,r_t,a_t}[r_t(\pi^*(x))-r_t(a_t)|\Gamma_{t-1}] =  \sum_{\pi\in\Psi}Q_{m(t)}(\pi)\Reg(\pi). $$
Hence from Azuma's inequality, with probability at least $1-\delta/2$, we have:
\begin{align}
\label{ineq:azuma-bound}
  \sum_{t\in\activesteps} M_t \leq 2\sqrt{2|\activesteps|\log(2/\delta)} \leq 2\sqrt{2T\log(2/\delta)}.
\end{align}
Hence when \Cref{ineq:azuma-bound} holds, we get:
\begin{align}
\label{ineq:cum-reg-bound}
\begin{split}
  &\sum_{t=1}^T \Big( r_t(\pi^*(x_t))-r_t(a_t) \Big)\\
  & = \sum_{t\in\passivesteps} \Big( r_t(\pi^*(x_t))-r_t(a_t) \Big) + \sum_{t\in\activesteps} \Big( r_t(\pi^*(x_t))-r_t(a_t) \Big) \\
  & \leq |\passivesteps| + \sum_{t\in\activesteps}\sum_{\pi\in\Psi} Q_{m(t)}(\pi)\Reg(\pi) + \sqrt{8T\log(2/\delta)}
\end{split}
\end{align}
Since in any epoch $m\geq 1$, there are at most $1 + \epsilon(\tau_m-\tau_{m-1})$ passive time-steps. Therefore:
\begin{align}
\label{ineq:bound-passive-steps}
  |\passivesteps| \leq \epsilon T + m(T) \leq 1 + \log_2(T) + \epsilon T.
\end{align}
Further when $\event$ holds, from \Cref{lem:QmRegTrue}, we have:
\begin{align}
\label{ineq:Qtrue-active}
\begin{split}
  &\sum_{t\in\activesteps}\sum_{\pi\in\Psi} Q_{m(t)}(\pi)\Reg(\pi) \\
  & \leq \tau_1 + \sum_{ \{t\in\activesteps \suchthat t \; \geq \tau_1+1\} }\sum_{\pi\in\Psi} Q_{m(t)}(\pi)\Reg(\pi)\\
  & \leq \tau_1 + C_{10}\sqrt{\frac{KB}{\sqrt{\epsilon^{\rho}}}}T + \sum_{t=\tau_1+1}^T \frac{(2+C_0)K}{\gamma_{m(t)}}\\
  & \leq \tau_1 + C_{10}\sqrt{\frac{KB}{\sqrt{\epsilon^{\rho}}}}T + \sum_{m=2}^{m(T)} \frac{(2+C_0)K}{\gamma_{m}}(\tau_m-\tau_{m-1})
\end{split}
\end{align}
Since $\tau_1\geq 4$, $\tau_{m(t)-1}\leq t$ for all $t\geq 1$, and $\tau_m=\tau_1 2^{m-1}$ for all $m\geq1$. We get that $\tau_{m(T)-1}\leq T$, $\tau_{m(T)}\leq 2T$, and $m-1 \leq \log_2(T)$ for all $m\leq m(T)$. Therefore, we get:
\begin{align}
\label{ineq:sum-of-inv-gamma}
\begin{split}
  &\sum_{m=2}^{m(T)} \frac{(2+C_0)K}{\gamma_{m}}(\tau_m-\tau_{m-1})\\
  & = \sum_{m=2}^{m(T)} (2+C_0)K \sqrt{\frac{\ln^{\rho'}(\tau_{m-1}-\tau_{m-2})\ln((m-1)/\delta)\comp(\F)}{C_3K(\tau_{m-1}-\tau_{m-2})^{\rho}}}(\tau_m-\tau_{m-1})\\
  & \leq \frac{(2+C_0)}{\sqrt{C_3}} \sqrt{K\ln^{\rho'}(T)\ln(\log_2(T)/\delta)\comp(\F)} \sum_{m=2}^{m(T)} \frac{\tau_m-\tau_{m-1}}{\sqrt{(\tau_{m-1}-\tau_{m-2})^{\rho}}}.
\end{split}
\end{align}
Since for all $m\geq 1$, $\tau_{m+1}=2\tau_m$, we have that:
\begin{align}
\label{ineq:integral-bound}
\begin{split}
  &\sum_{m=2}^{m(T)} \frac{\tau_m-\tau_{m-1}}{\sqrt{(\tau_{m-1}-\tau_{m-2})^{\rho}}} = 2^{\rho/2} \sum_{m=2}^{m(T)} \frac{\tau_m-\tau_{m-1}}{\tau_{m-1}^{\rho/2}} \leq 2^{\rho/2}\sum_{m=2}^{m(T)}\int_{\tau_{m-1}}^{\tau_m} \frac{dy}{y^{\rho/2}}\\
  & = 2^{\rho/2}\int_{\tau_{1}}^{\tau_{m(T)}} \frac{dy}{y^{\rho/2}} \leq \frac{2^{\rho/2}}{1-\rho/2} \tau_{m(T)}^{1-\rho/2} \leq \frac{2}{1-\rho/2} T^{1-\rho/2}.
\end{split}
\end{align}
Where the last inequality follows from the fact that $\tau_{m(T)}\leq 2T$. Hence when \Cref{ineq:azuma-bound} and $\event$ hold, from \Cref{ineq:cum-reg-bound}, \Cref{ineq:bound-passive-steps}, \Cref{ineq:Qtrue-active}, \Cref{ineq:sum-of-inv-gamma}, and \Cref{ineq:integral-bound}, we get:
\begin{align*}
\begin{split}
  &\sum_{t=1}^T \Big( r_t(\pi^*(x_t))-r_t(a_t) \Big)\\
  & \leq |\passivesteps| + \sum_{t\in\activesteps}\sum_{\pi\in\Psi} Q_{m(t)}(\pi)\Reg(\pi) + \sqrt{8T\log(2/\delta)}\\
  & \leq 1 + \log_2(T) + \epsilon T + \sqrt{8T\log(2/\delta)} + \tau_1 + C_{10}\sqrt{\frac{KB}{\sqrt{\epsilon^{\rho}}}} T  \\
  & \;\;\;\;\;\;\; + \frac{(2+C_0)}{\sqrt{C_3}}\frac{4}{2-\rho}  \sqrt{K T^{2-\rho}\ln^{\rho'}(T)\ln(\log_2(T)/\delta)\comp(\F)} \\
  & = \ordO\Bigg( \Bigg(\epsilon + \sqrt{\frac{KB}{\sqrt{\epsilon^{\rho}}}} \Bigg)T + \sqrt{K T^{2-\rho}\ln^{\rho'}(T)\ln(\log_2(T)/\delta)\comp(\F)} \Bigg).
\end{split}
\end{align*}
Note that from \Cref{lem:ConReg}, we know that $\event$ holds with probability $1-\delta/2$. Also from Azuma's inequality, we showed that \Cref{ineq:azuma-bound} holds with probability $1-\delta/2$. Hence from union bound, we get that the above inequality holds with probability $1-\delta$. This concludes the proof of \Cref{eq:main-theorem}.

\section{Learning rates}
\label{app:kol-fast-rates}

In this section, we restate results from \cite{koltchinskii2011oracle} on bounds for excess risk in a form that is convenient for us to use. We consider the standard machine learning setting. That is, we let $(Z,Y)$ be a random tuple in $\Zscript \times [0,1]$ with distribution $P$. Assume $Z$ is observable and $Y$ is to be predicted based on an observation of $Z$. Let $l:\R\times \R$ be the squared error loss, that is $l(a,b)=(a-b)^2$. Given a function $g:\Zscript\rightarrow\R$, let $(l\cdot g)(z,y):=l(y,g(z))$ be interpreted as the loss suffered when $g(z)$ is used to predict $y$. Let $\G$ be a convex class of functions from $\Zscript$ to $\R$. The problem of optimal prediction can be viewed as finding a solution to the following risk minimization problem:
$$ \min_{g\in\G} P(l\cdot g).$$
Where $P(l\cdot g)$ is a short hand for $\E_P[(l\cdot g)(Z,Y)]$. Let $\hatg^*\in\G$ be a solution to the above risk minimization problem. Let $g^*(z):=\E_P[Y|Z=z]$. Since the distribution $P$ is unknown, the above risk minimization problem is replaced by the empirical risk minimization problem:
$$ \min_{g\in\G} P_n(l\cdot g).$$
Where $P_n$ is an empirical distribution generated from $n$ i.i.d. samples of $(Z,Y)$ from the distribution $P$. Here $P_n(l\cdot g)$ is a short hand for $\E_{P_n}[(l\cdot g)(Z,Y)]$. In general, we will use $P(\cdot)$ and $P_n(\cdot)$ as a short hand for $\E_P[\cdot]$ and $\E_{P_n}[\cdot]$ respectively. Now, let $\hatg_n\in\G$ be a solution to the above empirical risk minimization problem. Also let $\G^l$ denote the loss class, that is $\G^l:=\{l\cdot g \suchthat g\in\G \}$. For any $g  \in \G$, we define the excess risk ($\erisk(l\cdot g)$) and the empirical excess risk ($\haterisk(l\cdot g)$), given by:
\begin{align*}
	&\erisk(l\cdot g):= P(l\cdot g) - \min_{l \cdot g'\in\G^l} P(l\cdot g') = \E_P[(l\cdot g)(Z,Y)] - \min_{g'\in\G}\E_P[(l\cdot g')(Z,Y)], \\
	&\haterisk(l\cdot g):= P_n(l\cdot g) - \min_{l \cdot g'\in\G^l} P_n(l\cdot g') = \E_{P_n}[(l\cdot g)(Z,Y)] - \min_{g'\in\G}\E_{P_n}[(l\cdot g')(Z,Y)].
\end{align*}
For $\delta\in\R_+$, we define the $\delta$-minimal set ($\G^l(\delta)$) and the empirical $\delta$-minimal set ($\hatG^l(\delta)$), given by:
\begin{align*}
	&\G^l(\delta) := \bigg\{ h\in\G^l \suchthat \erisk(h) \leq \delta \bigg\}, \;\;\; \hatG^l(\delta) := \bigg\{ h\in\G^l \suchthat \haterisk(h) \leq \delta \bigg\}.
\end{align*}
We now define a version of local Rademacher averages ($\psi_n$). We start by defining the Rademacher process ($R_n(\cdot)$). For any function $h:\Zscript\rightarrow\R$, $R_n(h)$ is given by:
$$R_n(h):=\frac{1}{n}\sum_{i=1}^n \epsilon_i h(Z_i).$$
Where $\{Z_i\}_{i=1}^n$ are i.i.d. random samples from the marginal distribution of $P$ on $\Zscript$. And where $\{\epsilon_i\}_{i=1}^n$ are i.i.d. Rademacher random variables (that is, $\epsilon_i$ takes the values $+1$ and $-1$ with probability $1/2$ each) independent of $Z_i$. We also define a (pseudo)-metric ($\rho_P$) on the set of functions that are square integrable with respect to $P$, such that: $\rho_P(f,g):=\sqrt{P((f-g)^2)}$. We now define the local Rademacher average ($\psi_n$) as:
\begin{align*}
  \psi_n(\delta) := 16\E_{P,\epsilon}\sup\{|R_n(g-\hatg^*)| \suchthat g\in\G, \rho_P^2(g,\hatg^*) \leq 2\delta  \}.
\end{align*}
Finally we define the $\flat$-transform and the $\sharp$-transform. For any $\kappa:\R_+\rightarrow\R_+$, define:
$$ \kappa^{\flat}(\delta):= \sup_{\delta'\geq\delta}\frac{\kappa(\delta')}{\delta'}, \;\;\; \kappa^{\sharp}(\epsilon):=\inf\{ \delta>0 \suchthat \kappa^{\flat}(\delta) \leq \epsilon \}.  $$
It is easy to see that $\sharp$-transforms are decreasing functions, and we will use this property in the proof of \Cref{lem:kol}. For more details and properties of these transformations, see section A.3 in \cite{koltchinskii2011oracle}. We now get to the main Lemma of this section (\Cref{lem:kol}), which is implicitly evident from results in \cite{koltchinskii2011oracle}. Lemma \ref{lem:kol} shows that, with high-probability, the $\delta$-minimal set ($\G^l(\delta)$) and the empirical $\delta$-minimal set ($\hatG^l(\delta)$) approximate each other.
\begin{lemma}
	\label{lem:kol}
  Let $\G$ be a convex class of functions from $\Zscript$ to $[0,1]$.
	Suppose $\zeta\in(0,1/2)$. With probability at least $1-\zeta$, for all $\delta\geq \max\{\psi_n^{\sharp}(\frac{1}{16}), \frac{16384\ln(2/\zeta)}{n}\}$ we have: $$\G^l(\delta) \subset \hatG^l(3\delta/2), \;\;\; \hatG^l(\delta)\subset \G^l(2\delta).$$
\end{lemma}
\begin{proof}
	Lemma~\ref{lem:kol} is a corollary of a few Lemmas and inequalities in \cite{koltchinskii2011oracle}. In the next few steps, we will define a function $U_n$ and bound $U_n^{\sharp}(1/2)$. Lemma~\ref{lem:kol} will follow from Lemma 4.2 in \cite{koltchinskii2011oracle} and the bounds on $U_n^{\sharp}(1/2)$. Let $D(\delta)$ denote the $\rho_P$-diameter of the $\delta$-minimal set ($\G^l(\delta)$). That is:
	$$ D(\delta):= \sup_{h,h'\in\G^l(\delta)}\rho_P(h,h'). $$
	Also let $\phi_n$ be a measure of empirical approximation:
	$$ \phi_n(\delta):= \E\bigg[ \sup_{h,h'\in\G^l(\delta)} \Big| (P_n-P)(h-h') \Big| \bigg]. $$
	Let $t,\sigma>0$, and $q>1$. We will fix the values of $t,\sigma$ and $q$ later in the proof. Let $\delta_j:=q^{-j}$ and $t_j:=t\frac{\delta_j}{\sigma}$, for all $j\geq0$. We will now define a function $U_n:(0,1]\rightarrow\R_+$. For all $j\geq 0$ and $\delta\in(\delta_{j+1},\delta_j]$, define:
	\begin{align*}
		 U_n(\delta) &:=\phi_n(\delta_j) + \sqrt{2\frac{t_j}{n}(D^2(\delta_j) + 2\phi_n(\delta_j))} + \frac{t_j}{2n}\\
		 & = \phi_n(\delta_j) + \sqrt{2\frac{t}{n}\frac{\delta_j}{\sigma}(D^2(\delta_j) + 2\phi_n(\delta_j))} + \frac{t}{2n}\frac{\delta_j}{\sigma}
	\end{align*}
	The reader may have astutely noticed that functions like $U_n$ appear as upper bounds in Talagrand type concentration inequalities, in fact that is where this comes from. We now bound $U_n^{\flat}(\eta)$ for all $\eta>0$:
	\begin{align}
  \label{ineq:U-flat-bound}
  \begin{split}
		&U_n^{\flat}(\eta) \leq  \sup_{\delta_j\geq \eta} \frac{q}{\delta_j} \bigg\{  \phi_n(\delta_j) + \sqrt{2\frac{t}{n}\frac{\delta_j}{\sigma}(D^2(\delta_j) + 2\phi_n(\delta_j))} + \frac{t}{2n}\frac{\delta_j}{\sigma} \bigg\}\\
		&\leq q\sup_{\delta_j\geq \eta} \frac{\phi_n(\delta_j)}{\delta_j} +  \sup_{\delta_j\geq \eta} q \bigg\{\sqrt{\frac{2t}{\sigma n}\frac{D^2(\delta_j)}{\delta_j}} + \sqrt{\frac{4t}{\sigma n}\frac{\phi_n(\delta_j)}{\delta_j}} \bigg\} + \frac{qt}{2\sigma n}  \\
		&\leq q\phi_n^{\flat}(\eta) +  q \sqrt{\frac{2t}{\sigma n}(D^2)^{\flat}(\eta)} + q\sqrt{\frac{4t}{\sigma n}\phi_n^{\flat}(\eta) }  + \frac{qt}{2\sigma n}
  \end{split}
	\end{align}
  From \Cref{ineq:U-flat-bound}, we get a bound on $U_n^{\sharp}(\epsilon)$ for all $\epsilon>0$:
	\begin{align}
  \label{ineq:U-sharp-bound}
  \begin{split}
		&U_n^{\sharp}(\epsilon) := \inf\{ \eta>0 \suchthat U_n^{\flat}(\eta) \leq \epsilon \}\\
		& \leq  \inf\bigg\{ \eta>0 \suchthat q\phi_n^{\flat}(\eta) +  q \sqrt{\frac{2t}{\sigma n}(D^2)^{\flat}(\eta)} + q\sqrt{\frac{4t}{\sigma n}\phi_n^{\flat}(\eta) }  + \frac{qt}{2\sigma n}  \leq \epsilon \bigg\}\\
		& \leq  \inf\bigg\{ \eta>0 \suchthat \phi_n^{\flat}(\eta) +  \sqrt{\frac{2t}{\sigma n}(D^2)^{\flat}(\eta)} + \sqrt{\frac{4t}{\sigma n}\phi_n^{\flat}(\eta) }   \leq \frac{1}{q}\bigg(\epsilon - \frac{qt}{2\sigma n}\bigg) \bigg\}\\
		& \leq \max\Bigg\{ \inf\bigg\{ \eta>0 \suchthat \phi_n^{\flat}(\eta) \leq \frac{1}{3q}\bigg(\epsilon - \frac{qt}{2\sigma n}\bigg) \bigg\},  \inf\bigg\{ \eta>0 \suchthat  \sqrt{\frac{2t}{\sigma n}(D^2)^{\flat}(\eta)}   \leq \frac{1}{3q}\bigg(\epsilon - \frac{qt}{2\sigma n}\bigg) \bigg\},\\
		&  \;\;\;\;\;\;\;\;\;\;\;\;\;\;\;\;\; \inf\bigg\{ \eta>0 \suchthat \sqrt{\frac{4t}{\sigma n}\phi_n^{\flat}(\eta) }   \leq \frac{1}{3q}\bigg(\epsilon - \frac{qt}{2\sigma n}\bigg) \bigg\} \Bigg\}\\
		& \leq \max\bigg\{\phi_n^{\sharp}\bigg(\frac{\epsilon}{3q} - \frac{t}{6\sigma n}\bigg), (D^2)^{\sharp}\bigg(\frac{\sigma n}{2t}\bigg(\frac{\epsilon}{3q} - \frac{t}{6\sigma n}\bigg)^2\bigg), \phi_n^{\sharp}\bigg(\frac{\sigma n}{4t}\bigg(\frac{\epsilon}{3q} - \frac{t}{6\sigma n}\bigg)^2\bigg) \bigg\}
  \end{split}
  \end{align}
  To further bound $U_n^{\sharp}(\cdot)$, we need to bound the terms in \Cref{ineq:U-sharp-bound}. From page 78 in \cite{koltchinskii2011oracle}, we get that the convexity of $\G$ implies a bound on $D(\cdot)$ which further gives us a bound on $(D^2)^{\flat}(\cdot)$:
	\begin{align*}
		&D(\delta) \leq 4\sqrt{2}\sqrt{\delta}, \;\;\; \text{for all $\delta\geq 0$}.\\
    \implies & (D^2)^{\flat}(\eta) = \sup_{\delta'\geq\eta}\frac{D^2(\delta')}{\delta'} \leq 32, \;\;\; \text{for all $\eta\geq 0$}.
	\end{align*}
  Hence we have:
	\begin{align}
  \label{eq:D2-sharp-bound}
		(D^2)^{\sharp}(\epsilon) = 0, \;\;\; \text{for all $\epsilon\geq 32$}
	\end{align}
  To upper-bound $U_n^{\sharp}(1/2)$, we now bound the $(D^2)^{\sharp}(\cdot)$ term in \Cref{ineq:U-sharp-bound}. To do this we choose $\sigma = 4096 tq^2/n$. Hence from the choice of $\sigma$ and from \Cref{eq:D2-sharp-bound}, we get:
	\begin{align}
  \label{eq:D2-sharp-bound-for-U-sharp}
  \begin{split}
		\sigma \geq \frac{4096 tq^2}{n} = 1024 \frac{t q^2}{n(1/2)^2} \implies \frac{\sigma n}{2t}\frac{(1/2)^2}{16q^2} \geq 32 \implies \frac{\sigma n}{2t}\bigg(\frac{1/2}{3q} - \frac{t}{6\sigma n}\bigg)^2 \geq 32\\
		\implies (D^2)^{\sharp}\bigg(\frac{\sigma n}{2t}\bigg(\frac{1/2}{3q} - \frac{t}{6\sigma n}\bigg)^2\bigg) = 0
  \end{split}
	\end{align}
  We now bound the $\phi_n^{\sharp}(\cdot)$ terms in \Cref{ineq:U-sharp-bound}, in terms of $\psi_n^{\sharp}(\cdot)$. Again from page 78 in \cite{koltchinskii2011oracle}, we get that the convexity of $\G$ implies a bound on $\phi_n(\cdot)$ which further gives us a bound on $\phi_n^{\sharp}(\cdot)$:
  \begin{align}
  \label{ineq:phi-sharp-bound}
  \begin{split}
 		&\phi_n(\delta) \leq \psi_n(\delta)  \;\;\; \text{for all $\delta\geq 0$}.\\
    \implies & \phi_n^{\sharp}(\epsilon) \leq \psi_n^{\sharp}(\epsilon)  \;\;\; \text{for all $\epsilon\geq 0$}.
  \end{split}
 	\end{align}
  To upper-bound $U_n^{\sharp}(1/2)$, we now bound the $\phi_n^{\sharp}(\cdot)$ terms in \Cref{ineq:U-sharp-bound}. From the choice of $\sigma$ and from \Cref{ineq:phi-sharp-bound}, we get:
  \begin{align}
  \label{ineq:phi-sharp-bound-for-U-sharp-1}
  \begin{split}
		\sigma  \geq \frac{4qt}{n} = \frac{2qt}{n(1/2)} \implies \frac{1/2}{12q} \geq \frac{t}{6\sigma n} \implies \frac{1/2}{3q} - \frac{t}{6\sigma n} \geq \frac{1/2}{4q}\\
    \implies \phi_n^{\sharp}\bigg(\frac{1/2}{3q} - \frac{t}{6\sigma n}\bigg) \leq \phi_n^{\sharp}\bigg(\frac{1/2}{4q}\bigg) \leq \psi_n^{\sharp}\bigg(\frac{1/2}{4q}\bigg).
  \end{split}
	\end{align}
  Again from the choice of $\sigma$ and from \Cref{ineq:phi-sharp-bound}, we get:
  \begin{align}
  \label{ineq:phi-sharp-bound-for-U-sharp-2}
  \begin{split}
		\sigma \geq \frac{32tq}{n} = \frac{16tq}{n(1/2)} \implies \frac{\sigma n}{4t}\frac{(1/2)^2}{16q^2} \geq \frac{1/2}{4q} \implies \frac{\sigma n}{4t}\bigg(\frac{1/2}{3q} - \frac{t}{6\sigma n}\bigg)^2 \geq \frac{1/2}{4q}\\
		\implies \phi_n^{\sharp}\bigg( \frac{\sigma n}{4t}\bigg(\frac{1/2}{3q} - \frac{t}{6\sigma n}\bigg)^2 \bigg) \leq \phi_n^{\sharp}\bigg(\frac{1/2}{4q}\bigg) \leq \psi_n^{\sharp}\bigg(\frac{1/2}{4q}\bigg).
  \end{split}
	\end{align}
  Combining \Cref{ineq:U-sharp-bound}, \Cref{eq:D2-sharp-bound-for-U-sharp}, \Cref{ineq:phi-sharp-bound-for-U-sharp-1}, and \Cref{ineq:phi-sharp-bound-for-U-sharp-2}, we get:
  \begin{align}
  \label{ineq:U-sharp-bound-main}
  \begin{split}
		U_n^{\sharp}(1/2) \leq \psi_n^{\sharp}\bigg(\frac{1}{8q}\bigg).
  \end{split}
	\end{align}
  Lemma 4.2 in \cite{koltchinskii2011oracle} states that with probability at least $1- \sum_{\delta_j\geq \delta_n^{\diamond}}e^{-t_j}$, for all $\delta\geq \delta_n^{\diamond}$ we have: $\G^l(\delta) \subset \hatG^l(3\delta/2)$ and $\hatG^l(\delta)\subset \G^l(2\delta)$. Where $\delta_n^{\diamond}$ is any number such that $\delta_n^{\diamond}\geq U_n^{\sharp}(1/2)$. Hence from \Cref{ineq:U-sharp-bound-main}, we can choose:
	$$ \delta_n^{\diamond} = \max\bigg\{ \psi_n^{\sharp}\bigg(\frac{1}{8q}\bigg), \frac{4096tq^2}{n} \bigg\} \geq \max\{ U_n^{\sharp}(1/2), \sigma \}.$$
  Now by choosing $q=2$ and $t=\ln(2/\zeta)$, using the fact that $\zeta\in(0,1/2)$, we get that $t\geq 1$. Hence, we have that:
	\begin{align*}
		&\sum_{\delta_j \geq \delta_n^{\diamond} } e^{-t_j} \leq \sum_{\delta_j \geq \sigma } e^{-t_j} = \sum_{\delta_j\geq \sigma} \exp\Big\{ -t \frac{\delta_j}{\sigma} \Big\} \leq \sum_{j\geq 0} e^{-t q^j} =\\
		& e^{-t} + \frac{q}{q-1}\sum_{j=1}^{\infty} q^{-j}(q^j - q^{j-1})e^{-tq^j} \leq e^{-t} + \frac{1}{q-1}\int_{1}^{\infty} e^{-tx}dx = \\
		& e^{-t} + \frac{1}{q-1}\frac{1}{t}e^{-t} \leq e^{-t} + \frac{1}{q-1}e^{-t} = \frac{q}{q-1}e^{-t} = \zeta .
	\end{align*}
  That is, we have shown that with probability at least $1-\zeta$, for all $\delta \geq \max\{\psi_n^{\sharp}(1/8q), 4096tq^2/n\}$, we have: $\G^l(\delta) \subset \hatG^l(3\delta/2)$ and $\hatG^l(\delta)\subset \G^l(2\delta)$.
\end{proof}
Corollary \ref{cor:kol-linear} uses \Cref{lem:kol} and a bound on $\psi_n^{\sharp}(\cdot)$ when $\G$ is a convex subset of a $d$-dimensional linear space to show that for all $\delta\geq \frac{Cd\ln(1/\zeta)}{n}$, the $\delta$-minimal set ($\G^l(\delta)$) and the empirical $\delta$-minimal set ($\hatG^l(\delta)$) approximate each other with probability at least $1-\zeta$.
\begin{corollary}
\label{cor:kol-linear}
    Let $\G$ be a convex class of functions from $\Zscript$ to $[0,1]$, and a subset $d$ dimensional linear space. Suppose $\zeta\in(0,1/2)$. With probability at least $1-\zeta$, for all $\delta\geq \frac{Cd\ln(1/\zeta)}{n}$ we have: $$\G^l(\delta) \subset \hatG^l(3\delta/2), \;\;\; \hatG^l(\delta)\subset \G^l(2\delta).$$
    Where $C>0$ is a positive constant.
\end{corollary}
\begin{proof}
Since $\G$ is a convex subset of a $d$ dimensional linear space, we get from proposition 3.2 in \cite{koltchinskii2011oracle} that:
\begin{align*}
  \psi_n(\delta) &= 16\E_{P,\epsilon}\sup\{|R_n(g-\hatg^*)| \suchthat g\in\G, \rho_P^2(g,\hatg^*) \leq 2\delta  \} \\
  & \leq 16 \sqrt{2\delta} \sqrt{\frac{d}{n}}.
\end{align*}
Which implies that:
\begin{align*}
  \psi_n^{\flat}(\delta)= \sup_{\delta'\geq\delta}\frac{\psi_n(\delta')}{\delta'} \leq \sup_{\delta'\geq\delta} 16\sqrt{\frac{2 d}{n\delta'}} = 16\sqrt{\frac{2d}{n\delta}}.
\end{align*}
Hence, we get that:
\begin{align*}
  \psi_n^{\sharp}(\epsilon)&=\inf\{ \delta>0 \suchthat \psi_n^{\flat}(\delta) \leq \epsilon \}\\
  &\leq \inf\bigg\{ \delta>0 \suchthat 16\sqrt{\frac{2d}{n\delta}} \leq \epsilon \bigg\}\\
  &=\inf\bigg\{ \delta>0 \suchthat \frac{512d}{n\epsilon^2} \leq \delta \bigg\} = \frac{512d}{n\epsilon^2}.
\end{align*}
Therefore:
\begin{align*}
  \psi_n^{\sharp}(1/16) \leq \frac{512d}{n (1/ 16)^2} = \frac{131072 d}{n}.
\end{align*}
Hence \Cref{cor:kol-linear} follows from \Cref{lem:kol} and the above inequality.
\end{proof}

\paragraph{Rates for general classes of functions}

Lemma \ref{lem:kol-psi-bound} provides rates for $\psi_n^{\sharp}$ for different classes of $\G$. Hence similar to \Cref{cor:kol-linear}, these bounds imply that for all $\delta\geq \ordO(\psi_n^{\sharp}(1/16)\ln(1/\zeta))$, the $\delta$-minimal set ($\G^l(\delta)$) and the empirical $\delta$-minimal set ($\hatG^l(\delta)$) approximate each other with probability at least $1-\zeta$. The results stated in \Cref{lem:kol-psi-bound} are from \cite{koltchinskii2011oracle} (pages 85 to 87), we state the same results without proof.

\begin{lemma}
\label{lem:kol-psi-bound}
Let $\G$ be a convex class of functions from $\Zscript$ to $[0,1]$.
\begin{itemize}[noitemsep,topsep=0pt]
    \item Suppose $\G$ is VC-subgraph class of functions with VC-dimension $V$. Then for all $\epsilon>0$, we have:
    $$ \psi_n^{\sharp}(\epsilon) \leq \ordO\bigg(\frac{V}{n\epsilon^2}\log\bigg(\frac{n\epsilon^2}{V}\bigg) \bigg).$$
    \item Let $N(\G, L_2(P_n), \epsilon)$ denote the number of $L_2(P_n)$ balls of radius $\epsilon$ covering $\G$. Suppose the empirical entropy is bounded, that is for some $\rho\in(0,1)$ we have that: $\log(N(\G, L_2(P_n), \epsilon))\leq \ordO(\epsilon^{-2\rho})$. Then for all $\epsilon>0$, we have:
    $$ \psi_n^{\sharp}(\epsilon) \leq \ordO\Big((n\epsilon^2)^{\frac{-1}{1+\rho}}\Big). $$
    \item Suppose $\G$ is a convex hull of a VC-subgraph class of functions with VC-dimension $V$. Then for all $\epsilon>0$, we have:
    $$ \psi_n^{\sharp}(\epsilon) \leq \ordO\Bigg(\bigg(\frac{V}{n\epsilon^2}\bigg)^{\frac{1}{2}\frac{2+V}{1+V}}\Bigg).$$
\end{itemize}
\end{lemma}

\paragraph{Proving \Cref{ass:main-assumption}}

We now describe the general outline to prove \Cref{ass:main-assumption} using the results in this section for different convex classes $\F$. Note that we need the conditions of \Cref{ass:main-assumption} to hold for any convex set $\F'\subseteq \F$, and any action selection kernel $p$. First let $\Zscript$ used in this section correspond to $\Xscript\times\A$, and let distribution $P$ correspond to the distribution described by $x_t \sim \mathcal{D}_\Xscript$, $a|x \sim p(a|x)$ and $r_t \sim \mathcal{D}_{r|x,a}$ induced by the action selection kernel $p$. Also note that the empirical distribution corresponding to $\Stilde$, in fact corresponds to $P_n$ in this section. Hence from \cref{lem:kol}, to show that the empirical and population $\eta$-minimal sets approximate each other with high-probability (as is required in \Cref{ass:main-assumption}), it is sufficient to bound $\psi_n^{\sharp}(1/16)$ uniformly for all convex subsets $\F'\subseteq \F$ and all distributions induced by action selection kernels. Such bounds can be proven for many interesting convex classes of estimators because the bounds on $\psi_n^{\sharp}$ are often distribution-free and we often have that $\comp(\F')\leq \comp(\F)$. 

For example, say $\F$ is a convex subset of a $d$ dimensional linear space, then any convex subset $\F'\subseteq\F$ is also a convex subset of a $d$ dimensional linear space. Hence, \cref{cor:kol-linear} can be used on $\F'$ to show that the empirical and population $\eta$-minimal sets approximate each other (as is required in \Cref{ass:main-assumption}). Note that this along with \Cref{thm:main-theorem} gives us \Cref{thm:linear-theorem}. Similarly, say $\F$ is a convex set with VC sub-graph dimension $V$. Note that, for any convex set $\F'\subseteq\F$, we have that $\F'$ has a VC sub-graph dimension $V$. Hence, we can then use \Cref{lem:kol-psi-bound} to bound $\psi_n^{\sharp}(1/16)$ in a distribution free manner and then show that the empirical and population $\eta$-minimal sets approximate each other (using \Cref{lem:kol}). Note that this along with \Cref{thm:main-theorem} gives us Example 1 in \Cref{sec:main_results}. We can similarly that Examples 2 and 3 follow from \Cref{thm:main-theorem} and the results in this section.


\section{Solving the constrained regression problem}
\label{app:conreg-tractability}

In this section, we show the constrained regression problem can be solved using a weighted regression oracle. The purpose of this argument is to show that the constrained regression problem is computationally tractable for many class of estimators. Suppose $\F$ is a convex set. Let $S, S'\subseteq \Xscript\times\A\times[0,1]$, often these sets represent the data collected in the active and passive phases respectively. Consider the following optimization problem:
\begin{equation}
\label{opt:genreral-conreg}
\begin{aligned}
\min_{f\in\F} \quad & \frac{1}{|S|}\sum_{(x,a,r(a))\in S} (f(x,a)-r(a))^2\\
\textrm{s.t.} \quad & \frac{1}{|S'|}\sum_{(x,a,r(a))\in S'} (f(x,a)-r(a))^2 \leq \alpha + \beta. \\
\end{aligned}
\end{equation}
Where $\beta>0$ is a fixed problem parameter, and $\alpha := \frac{1}{|S'|}\min_{f\in\F} \sum_{(x,a,r(a))\in S'} (f(x,a)-r(a))^2$. From the definition of $\alpha$ and $\beta$, we have that there exists a $g\in\F$ such that:
\begin{align}
\label{regurality-condition}
  \frac{1}{|S'|}\sum_{(x,a,r(a))\in S'} (g(x,a)-r(a))^2 < \alpha + \beta. 
\end{align}
That is there is a $g\in\F$ such that the constraint in the optimization problem (\ref{opt:genreral-conreg}) is not tight. Hence strong duality holds \footnote{See proposition 1.1.3 in \cite{bertsekas2015convex}.}. 
Now consider the lagragian of the constrained regression problem:
\begin{align*}
  L(f,\lambda) := \frac{1}{|S|}\sum_{(x,a,r(a))\in S} (f(x,a)-r(a))^2 + \lambda \Bigg(\frac{1}{|S'|}\sum_{(x,a,r(a))\in S'} (f(x,a)-r(a))^2 - \alpha - \beta\Bigg).
\end{align*}
Note that probelem \ref{opt:genreral-conreg} can be re-written as, $\min_{f\in\F}\max_{\lambda \geq 0} L(f,\lambda)$. Since strong duality holds, this is equivalent to solving the following dual optimization problem:
\begin{align*}
  \max_{\lambda \geq 0} \min_{f\in\F} L(f,\lambda) \equiv \max_{\lambda \geq 0} g(\lambda).
\end{align*}
Where, $g(\lambda):=\min_{f\in\F} L(f,\lambda)$. For any fixed $\lambda$, note that evaluating $g(\lambda)$ is equivalent to solving a weighted regression problem:
\begin{align*}
  \arg\min_{f\in\F} L(f,\lambda) = \arg\min_{f\in\F} \frac{1}{|S|}\sum_{(x,a,r(a))\in S} (f(x,a)-r(a))^2 + \frac{\lambda}{|S'|} \sum_{(x,a,r(a))\in S^{pass}} (f(x,a)-r(a))^2.
\end{align*}
Now, let $\lambda^*$ be an optimal dual solution. Since the dual problem is a one-dimensional concave maximization problem, we can use a bisection method to find the optimal dual solution. Hence one can solve the dual optimization problem with $\ordO(\log(\lambda^*))$ calls to evaluate $g(\cdot)$, where each evaluation call corresponds to one call to a weighted regression oracle. Suppose this procedure outputs $\bar{\lambda}$ as the optimal dual solution. We then output the estimator that solves:
\begin{align*}
   \arg\min_{f\in\F} \frac{1}{|S|} \sum_{(x,a,r(a))\in S} (f(x,a)-r(a))^2 + \frac{\bar{\lambda}}{|S'|} \sum_{(x,a,r(a))\in S^{pass}} (f(x,a)-r(a))^2.
\end{align*}
Note that this estimator must be optimal for the primal problem \footnote{Here when we say optimal, we mean optimal up to the accuracy thresholds.}. Since there are many algorithms and heuristics to solve weighted regression problems, this argument shows that the constrained regression problem is often computationally tractable. 
\begin{algorithm}[ht]
  \caption{Solving constrained regression}
  \label{alg:optimize-conreg}
  \textbf{input:} Given a threshold parameter $\kappa>0$ and a weighted regression oracle to evaluate $g(\cdot)$.
  \begin{algorithmic}[1] 
  \State Set $\lambda_L = 0$, $\lambda_M=1$, and $\lambda_R=2$.
  \While{$g(\lambda_M)<g(\lambda_R)$}
    \State Set $\lambda_R\leftarrow 2\lambda_R$ and set $\lambda_M \leftarrow 2\lambda_M$.
  \EndWhile
  \While{$|\lambda_R-\lambda_L|\geq \kappa$}
    \If{$g(\lambda_M+\kappa)>g(\lambda_M)$}
        \State Set $\lambda_L\leftarrow \lambda_M$.
    \Else
        \State Set $\lambda_R \leftarrow \lambda_M$.
    \EndIf
    \State Set $\lambda_M \leftarrow \frac{1}{2}(\lambda_L+\lambda_R)$.
  \EndWhile
  \State Return the output of the weighted regression oracle on the following problem:
  $$\min_{f\in\F} \frac{1}{|S|} \sum_{(x,a,r(a))\in S} (f(x,a)-r(a))^2 + \frac{\lambda_M}{|S'|} \sum_{(x,a,r(a))\in S^{pass}} (f(x,a)-r(a))^2.$$
  \end{algorithmic}
\end{algorithm}

We note that in practice, rather than solving multiple weighted regression problems, one may prefer to directly find a minimax solution to the lagrangian of the constrained regression problem (see \cite{jin2019local}).

\section{Sensitivity of confidence intervals to realizability}
\label{app:sensitivity-example}

In this section, we demonstrate that the confidence intervals used by LinUCB can be extremely sensitive to the realizability assumption. We also point out analogous issues in LinTS and FALCON (with linear estimates). We do this by constructing a family of contextual bandit problems where the approximation error to the class of linear models can be arbitrarily small, but given data from the policy induced by the best linear estimate (which also happens to be optimal), the confidence intervals used by LinUCB tightly concentrate around bad estimators that induce high-regret policies.

\begin{figure}[ht]
\vspace{.05in}
\centering
\includegraphics[width=.5\textwidth]{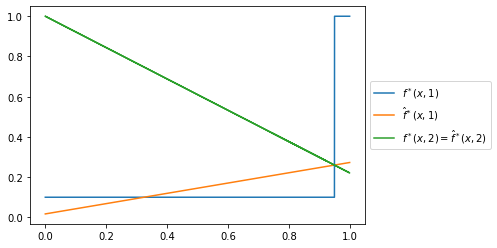}
\caption{This is a plot of the conditional expected reward ($f^*$) and the best linear estimate $\hatf^*$ when actions are sampled uniformly at random. Note that the conditional expected reward for arm 2 is linear. The problem is constructed so that the policy ($\pi_{\hatf^*}$) that is induced by the best linear estimates ($\hatf$) samples arm 2 for all $x$ such that $f^*(x,1)=0.1$, and samples arm 1 for all $x$ such that $f^*(x,1) = 1$. Note that this policy is also optimal.}
\vspace{.05in}
\label{fig:plot-hatf-star-policy}
\end{figure}

Consider a family of two armed contextual bandit problems that are parameterized by $\theta\in(0,0.05]$. Let $\Xscript=(0,1)$ be the set of contexts, and let $\A=\{1,2\}$ be the set of actions. At every time-step, the environment draws a context according to the continuous uniform distribution on $\Xscript$. That is, $D_{\Xscript}\equiv\Unif(\Xscript)$. To estimate the conditional expected reward ($f^*$) and select a policy, we pick estimators from a convex class of functions $\F$, where:
$$ \F :=\{ f:\Xscript\times\A \rightarrow \R \suchthat \text{ $f(\cdot,1)$ and $f(\cdot,2)$ are linear} \}. $$

For $\theta = 0.05$, \Cref{fig:plot-hatf-star-policy} plots the conditional expected rewards ($f^*$) and the best linear estimate $\hatf^*\in\F$ when actions are sampled uniformly at random. We will now specify these terms more generally, starting with the conditional expected reward for arm~1, which is given by:
$$
f^*(x,1):=
\begin{cases}
0.1, \; \text{for all $x\leq 1-\theta$}\\
1, \;\;\;\; \text{for all $x> 1-\theta$}.
\end{cases}
$$
The conditional expected reward for arm~2 is linear, and is given by  $f^*(x,2):=1+m_{\theta}x$. Where $m_{\theta}$ is such that $\hatf^*(x,1)$ and $f^*(x,2)$ meet at $x=1-\theta$, which is ensured by defining:
$$ m_{\theta} := \frac{\hatf^*(1-\theta,1)-1}{1-\theta}. $$
Since $f^*(\cdot,2)$ is linear, we get that $\hatf^*(\cdot,2)\equiv f^*(\cdot,2)$. Further since $m_{\theta}<0$, we get that $\hatf^*(x,2)$ is decreasing in $x$. Similarly, one can show that $\hatf^*(x,1)$ is increasing in $x$. Therefore, we get that $\pi_{\hatf^*}$ is given by:
$$
\pi_{\hatf^*}(x):=
\begin{cases}
2, \; \text{for all $x\leq 1-\theta$}\\
1, \; \text{for all $x> 1-\theta$}.
\end{cases}
$$
It is interesting to note that $\pi_{\hatf^*}$ is optimal for this family of bandit problems, that is $\pi_{f^*}\equiv \pi_{\hatf^*}$. Now let $\hatf$ be the best predictor of arm rewards under the distribution induced by $\pi_{\hatf^*}$. That is:
$$ \hatf \in \arg\min_{f\in\F} \E_{x\sim D_{\Xscript}}[(f(x,\pi_{\hatf^*}(x)) - f^*(x,\pi_{\hatf^*}(x)))^2]. $$
Since $f^*(x,1)=1$ for all $x>1-\theta$, we get that $\hatf(\cdot,1)\equiv 1$. Also since $f^*(\cdot,2)$ is linear and arm~2 is chosen for all $x\leq 1-\theta$, we get that $\hatf(\cdot,2)\equiv f^*(\cdot,2)$. For a more visual understanding, see \Cref{fig:plot-hatf-policy} which plots $f^*$ and $\hatf$ for $\theta = 0.05$.
\begin{figure}[ht]
\vspace{.05in}
\centering
\includegraphics[width=.5\textwidth]{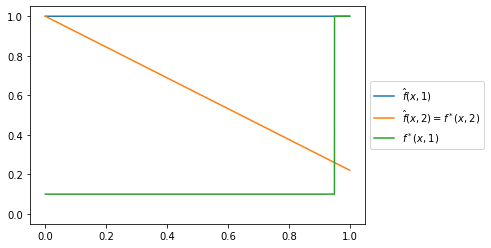}
\caption{This is a plot of the conditional expected reward ($f^*$) and the linear estimate ($\hatf$) that is learnt from data collected by $\pi_{\hatf^*}$. Note that $\pi_{\hatf^*}$ is infact the same as the optimal policy $\pi_{f^*}$. Also note that the policy $\pi_{\hatf}$ that is induced by the estimate $\hatf$ samples arm 1 for all $x$. Hence, this policy has high regret.}
\vspace{.05in}
\label{fig:plot-hatf-policy}
\end{figure}

Therefore $\pi_{\hatf}(x)=1$ for all $x$, and hence incurs high regret:
$$\Reg(\pi_{\hatf}) \geq \frac{1}{2}(1-\theta)(1-0.1) \geq 0.4275. $$
For this family of bandit problems, while the regret of $\pi_{\hatf}$ is at least $0.4275$,  the approximation error ($b$) can be arbitrarily small. In particular, since $f^*(\cdot,2)$ is linear, we get:
$$ b = \min_{f\in\F} \frac{1}{2}\E_{x\sim D_{\Xscript}}[(f(x,1) - f^*(x,1))^2] \leq \frac{1}{2}\E_{x\sim D_{\Xscript}}[(0.1 - f^*(x,1))^2] \leq \frac{\theta}{2}.$$
Further note that for this family of problems, as sufficient data is collected from policy $\pi_{\hatf^*}$ (which is also optimal), the confidence intervals used by LinUCB tightly concentrate around $\hatf$. 

Hence even under minor violations of realizability (the approximation error $b$ of the best linear estimator can be arbitrarily small), the confidence intervals that are used by LinUCB are invalid, in the sense that this confidence interval tightly concentrates on a bad linear estimate ($\hatf$) that induces a policy ($\pi_{\hatf}$) with high regret ($\Reg(\pi_{\hatf})>0.4275$). Note that a similar argument can be used to argue that for this family of bandit problems, given data from the optimal policy, the posterior of LinTS concentrates on the same bad linear estimate. Similarly for this family of bandit problems, given data from the optimal policy, the empirical risk minimizer would be the bad linear estimate $\hatf$ and the induced randomized policy constructed by FALCON would converge to the high regret policy ($\pi_{\hatf}$) induced by this estimate. This example calls into question the validity of any model update step in realizability-based approaches.

\end{document}